\documentclass{article}

\usepackage{microtype}
\usepackage{graphicx}
\usepackage{subcaption}
\usepackage{booktabs}

\usepackage{hyperref}

\usepackage[accepted]{icml2026}
\usepackage{comment}
\usepackage{amsmath}
\usepackage{amssymb}
\usepackage{mathtools}
\usepackage{amsthm}
\usepackage{algorithm}

\usepackage[capitalize,noabbrev]{cleveref}

\theoremstyle{plain}
\newtheorem{theorem}{Theorem}[section]
\newtheorem{proposition}[theorem]{Proposition}
\newtheorem{lemma}[theorem]{Lemma}

\theoremstyle{definition}

\theoremstyle{remark}

\newcommand{\dout}{d_{\mathrm{out}}}
\newcommand{\din}{d_{\mathrm{in}}}
\newcommand{\Sigx}{\Sigma_x}
\newcommand{\DWT}{\Delta W_T}
\newcommand{\tWT}{\widetilde{W}_T}
\newcommand{\Ut}{\widetilde{U}_T}
\newcommand{\Utr}{\widetilde{U}_{T,r}}
\newcommand{\sigt}{\widetilde{\sigma}}
\newcommand{\Lalign}{\mathcal{L}_{\mathrm{align}}}
\newcommand{\Lcoeff}{\mathcal{L}_{\mathrm{coeff}}}
\newcommand{\Lkd}{\mathcal{L}_{\mathrm{KD}}}

\newcommand{\colspan}{\operatorname{span}}

\icmltitlerunning{SAD-LoRA: Spectral Alignment for Low-Rank Knowledge Distillation}

\begin{document}

\twocolumn[
  \icmltitle{SAD-LoRA: Spectral Alignment for Low-Rank Knowledge Distillation}

  \icmlsetsymbol{equal}{*}

  \begin{icmlauthorlist}
    \icmlauthor{Omer Tariq}{yyy}
    \icmlauthor{Syed Muhammad Raza}{yyy}
    \icmlauthor{Jeongbae Son}{yyy}
  \end{icmlauthorlist}

  \icmlaffiliation{yyy}{Neubility Inc, Seoul, South Korea}

  \icmlcorrespondingauthor{Omer Tariq}{omer.tariq@neubility.co.kr}

  \icmlkeywords{Parameter-efficient fine-tuning, Knowledge distillation, Low-rank adaptation, Spectral methods, Model compression}

  \vskip 0.3in
]

\printAffiliationsAndNotice{}  

\begin{abstract}
Distilling a fine-tuned teacher into a LoRA-adapted student is a standard recipe for parameter-efficient compression, but output-level KD does not explicitly control which rank-$r$ weight subspace the adapter occupies. We propose \textbf{SAD-LoRA} (\textbf{S}pectral \textbf{A}lignment \textbf{D}istillation), which selects this subspace from the data-weighted student-space reference update $\DWT\Sigx^{1/2}$ and maintains it during training via a differentiable principal-angle loss on $\colspan(B)$. We show that the data-weighted distillation error decomposes exactly into subspace misalignment, within-subspace coefficient mismatch, and irreducible rank residual; standard KD can affect the first term only indirectly through output gradients. On controlled synthetic problems with a flat teacher spectrum, SAD-LoRA reduces the subspace-misalignment term from $51\%$ to nearly zero and lifts final subspace alignment from $0.49$ to $1.00$. On RoBERTa-large to RoBERTa-base distillation across six GLUE tasks, SAD-LoRA improves rank efficiency: at $r{=}4$, it matches or beats the strongest included spectral baseline on five of six tasks, and at $r{=}8$ it gives the best result on SST-2 and CoLA. Ablations identify subspace alignment as the load-bearing component, while coefficient matching is auxiliary.
\end{abstract}

\section{Introduction}

Parameter-efficient fine-tuning has become standard for adapting large pretrained models. Among existing approaches \citep{han2024peftsurvey}, low-rank adaptation (LoRA) is particularly attractive because it freezes pretrained weights and learns a task-specific update $\Delta W=BA$, with $B\in\mathbb{R}^{\dout\times r}$, $A\in\mathbb{R}^{r\times\din}$, and $r\ll\min(\dout,\din)$ \citep{hu2022lora}. In parallel, knowledge distillation (KD) compresses a stronger teacher into a smaller or capacity-constrained student \citep{hinton2015distilling,gou2021knowledge}. Combining LoRA with KD is therefore natural for efficient compression: the student receives teacher supervision while only a low-rank update is trained \citep{azimi2024kdlora,yang2024llmneo}.

This combination, however, leaves a structural question unresolved. Standard KD objectives operate on logits or intermediate activations, while LoRA constrains the student to a rank-$r$ \emph{weight} subspace. The KD loss can reduce output discrepancy without explicitly determining which $r$ directions the adapter should occupy. Because the rank-$r$ student cannot represent every direction of the teacher update, a misaligned column space of $B$ wastes part of the rank budget on directions that are off-target with respect to the teacher's task-specific transformation. We refer to this failure mode as \emph{subspace misalignment}: a distillation error component that is invisible to output-level KD but irreducible without rotating the adapter's column space.

Existing spectral LoRA methods address adjacent but distinct problems: they initialize or allocate adapters using the spectrum of the pretrained weight matrix \citep{meng2024pissa,wang-etal-2025-milora,buehler2024olora}, activation-aware pretrained-weight decompositions \citep{yang2024corda}, or sensitivity-based subset selection \citep{zhong2024seeking}. These methods provide useful starting points, but they do not explicitly preserve a teacher-update subspace during distillation; after initialization, $\colspan(B)$ can still drift under task and KD gradients. The central issue is therefore not only how to initialize LoRA, but how to maintain the relevant low-rank subspace throughout training (Fig.~\ref{fig:sadlora_concept}).

\begin{figure}[t]
    \centering
    \includegraphics[width=\columnwidth]{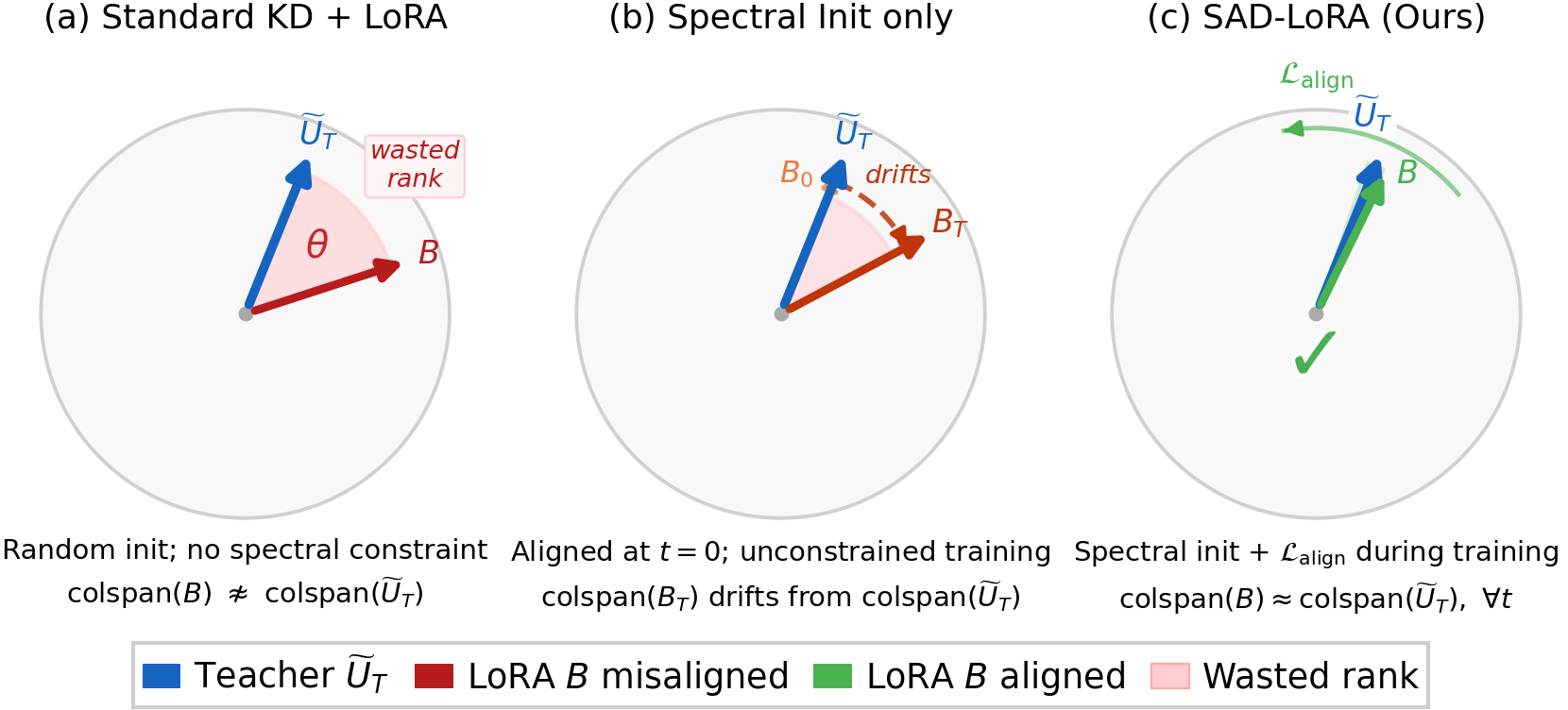}
    \caption{Three regimes for low-rank distillation.
    (a) Standard KD+LoRA may place $\colspan(B)$ away from the teacher subspace $\colspan(\Utr)$, wasting rank budget.
    (b) Spectral initialization aligns the adapter at $t=0$ but does not constrain later subspace drift.
    (c) SAD-LoRA combines spectral initialization with a continuous alignment loss $\Lalign$ that encourages $\colspan(B)\approx\colspan(\Utr)$ throughout training.}
    \label{fig:sadlora_concept}
    \vspace{-0.3cm}
\end{figure}

This paper studies LoRA distillation from a spectral approximation perspective. Let $W_0^{(\ell)}$ be a pretrained student weight matrix and let $W_T^{(\ell)}$ denote a task-adapted reference weight in the same student parameter space, defining the student-space update $\DWT^{(\ell)} = W_T^{(\ell)} - W_0^{(\ell)}$. A rank-$r$ LoRA student approximates this update by $B^{(\ell)}A^{(\ell)}$. Approximation quality depends not only on the singular values of $\DWT^{(\ell)}$, but also on whether $\colspan(B^{(\ell)})$ aligns with the dominant output-side directions of the update under the data distribution. Because distillation is evaluated on data, the relevant spectral object is the data-weighted update $\tWT^{(\ell)} = \DWT^{(\ell)}(\Sigma_x^{(\ell)})^{1/2}$, where $\Sigma_x^{(\ell)}$ is the layer-input covariance.

We propose SAD-LoRA, which turns alignment with the leading left singular subspace of $\tWT^{(\ell)}$ into a training objective. SAD-LoRA augments KD with a Grassmannian principal-angle loss on $\colspan(B)$ and, optionally, a within-subspace coefficient-matching term. We show that the data-weighted distillation error decomposes exactly into subspace misalignment, coefficient mismatch, and rank residual; at fixed rank, only the first two are reducible by changing $B$ and $A$. Standard KD can affect the subspace term only indirectly through output gradients, whereas SAD-LoRA targets it explicitly.

\textbf{Contributions.}
\begin{enumerate}
    \setlength{\itemsep}{0.15em}
    \setlength{\topsep}{0.25em}
    \item We formulate LoRA-based knowledge distillation as a data-weighted low-rank approximation problem and prove an exact three-term error decomposition (Theorem~\ref{thm:decomp}) that isolates subspace misalignment as a distinct, reducible source of error.
    \item We identify the optimal rank-$r$ adapter subspace as $\colspan(\Utr^{(\ell)})$, the leading left singular subspace of $\DWT^{(\ell)}(\Sigma_x^{(\ell)})^{1/2}$, and derive a layerwise rank-sufficiency criterion from its spectral tail (Propositions~\ref{prop:data_weighted_subspace}--\ref{prop:rank_suff}).
    \item We introduce SAD-LoRA, which uses a differentiable principal-angle alignment loss $\Lalign$ and an optional coefficient-matching loss $\Lcoeff$ to keep the LoRA adapter aligned with the data-weighted teacher-update subspace during distillation.
    \item We validate SAD-LoRA on controlled synthetic problems, where the decomposition can be checked exactly, and on RoBERTa-large to RoBERTa-base distillation across six GLUE tasks, emphasizing rank efficiency and alignment-versus-coefficient ablations.
\end{enumerate}

\begin{figure*}[t]
    \centering
    \includegraphics[width=0.93\textwidth]{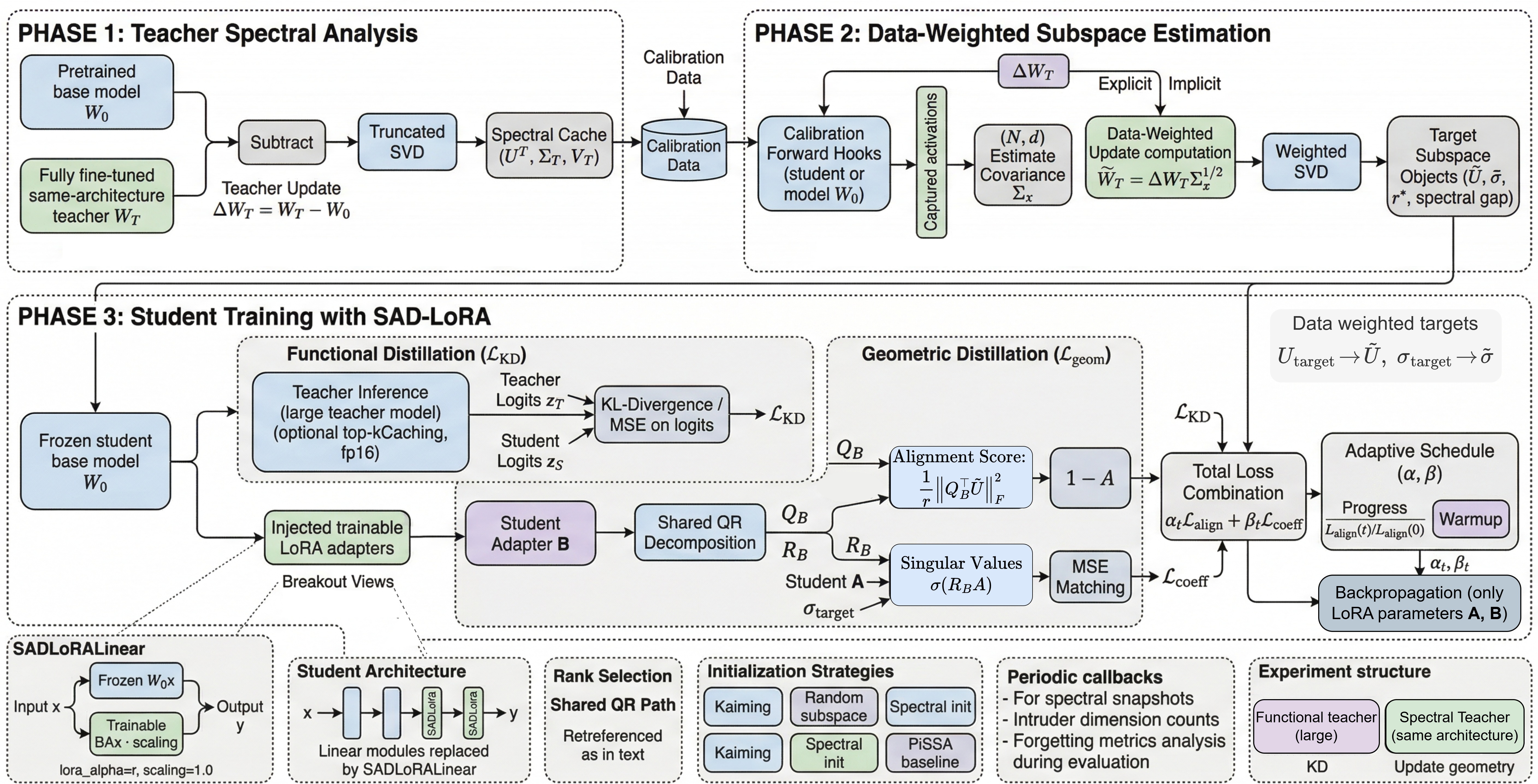}
    \caption{Overview of SAD-LoRA. Layerwise spectral targets $\Utr^{(\ell)}$ and $\sigt_T^{(\ell)}$ are computed offline from a same-architecture reference teacher and calibration activations, then frozen. During training, the student receives output-level KD while the LoRA column space is geometrically aligned to the frozen targets. Only the LoRA adapters $\{B^{(\ell)},A^{(\ell)}\}$ are optimized.}
    \label{fig:sadlora_overview}
    \vspace{-0.3cm}
\end{figure*}

\section{Related Work}
\label{sec:related_work}

\textbf{PEFT and LoRA.}
Parameter-efficient fine-tuning adapts pretrained models by training small task-specific modules on top of frozen or near-frozen backbones \citep{houlsby2019adapter,li2021prefix,lester2021prompt,benzaken2021bitfit,hu2022lora,ding2023peftsurvey,han2024peftsurvey}. LoRA is widely used because it parameterizes a frozen-weight update as $\Delta W=BA$, preserves the merged architecture, and adds no inference latency \citep{hu2022lora}. Subsequent variants improve LoRA through adaptive rank allocation, weight decomposition, quantization, parameter sharing, and modified parameterizations \citep{zhang2023adalora,valipour2022dylora,liu2024dora,kopiczko2024vera,renduchintala2023tiedlora,li2024vblora,dettmers2023qlora,li2023loftq,xu2023qalora}. SAD-LoRA keeps the same low-rank adapter form, but changes the training objective by explicitly aligning the adapter subspace with a data-weighted teacher-update target.

\textbf{Spectral and activation-aware LoRA.}
A closer line of work exploits spectral structure to improve LoRA. AdaLoRA uses an SVD-like parameterization and prunes less important singular components \citep{zhang2023adalora}; PiSSA initializes adapters from principal singular components of the pretrained weight $W_0$ \citep{meng2024pissa}; MiLoRA uses minor components to preserve principal pretrained knowledge \citep{wang-etal-2025-milora}; and OLoRA uses QR-based orthonormal initialization for optimization stability \citep{buehler2024olora}. These methods show that adapter subspace matters, but they decompose the pretrained weight or its compressed approximation, not the teacher's task-specific update. CorDA is closest to our data-weighted view: it orients pretrained-weight decomposition using input activation covariance from calibration data \citep{yang2024corda}. Activation-aware quantization and low-rank reconstruction use related covariance-aware principles \citep{yuan2023asvd,lin2024awq}. SAD-LoRA differs in the decomposed object: CorDA uses $W_0\widehat{\Sigma}_x$, whereas SAD-LoRA uses the student-space teacher update $\DWT\widehat{\Sigma}_x^{1/2}$. For KD, this distinction is central because the approximation target is the teacher's task update rather than the pretrained weight.

\textbf{KD with LoRA students.}
KD transfers teacher behavior through labels, logits, or intermediate representations \citep{hinton2015distilling,gou2021knowledge}. Recent methods combine KD with LoRA, including KD-LoRA for encoder benchmarks \citep{azimi2024kdlora}, LLM-Neo for parameter-efficient LLM distillation \citep{yang2024llmneo}, and PC-LoRA for progressive compression \citep{hwang2024pclora}. These works establish KD$+$LoRA as an effective compression paradigm, but treat the adapter mainly as a parameter-efficient training channel; they do not specify which rank-$r$ weight subspace should approximate the teacher update. Concurrent theory on low-rank KD provides sample-complexity rank guarantees \citep{soarez2026demystifying}; our criterion is complementary because it is layerwise, data-dependent, and based on the spectral tail of $\tWT$.

\textbf{Subspace drift and intruder dimensions.}
\citet{shuttleworth2024lora} show that LoRA trained with standard objectives can develop intruder dimensions: high-energy directions of $BA$ that are orthogonal to both the pretrained spectrum and the teacher update and can drive forgetting. SAD-LoRA addresses the same geometric failure mode from the KD side: components of $B$ outside the target data-weighted subspace increase $\Lalign$, so any output-level gain must trade off against the alignment penalty.
\section{Method}
\label{sec:method}

SAD-LoRA aligns the LoRA adapter with the data-weighted spectral subspace of a student-space teacher update. We first define the target update and its implicit SVD, then introduce the alignment and coefficient losses, the training objective, initialization, and computational cost. Fig.~\ref{fig:sadlora_overview} gives an overview; full pseudocode is provided in Appendix~\ref{app:algorithm}.

\subsection{Setup and the Data-Weighted Teacher Update}

Let $W_0^{(\ell)} \in \mathbb{R}^{\dout\times \din}$ denote the pretrained weight matrix at layer~$\ell$. A LoRA student freezes $W_0^{(\ell)}$ and learns a rank-$r_\ell$ update
\begin{equation}
    W_S^{(\ell)} = W_0^{(\ell)} + B^{(\ell)}A^{(\ell)},
    \qquad
    r_\ell \ll \min(\dout,\din),
\end{equation}
with $B^{(\ell)} \in \mathbb{R}^{\dout\times r_\ell}$ and $A^{(\ell)} \in \mathbb{R}^{r_\ell\times \din}$. The LoRA scaling factor is absorbed into $B^{(\ell)}A^{(\ell)}$. We suppress the layer index $(\ell)$ when context is clear. Let $W_T^{(\ell)}$ denote a task-adapted reference weight in the same parameter space as the student, and define the student-space reference update $\DWT^{(\ell)} = W_T^{(\ell)} - W_0^{(\ell)}$. When the functional teacher providing distillation logits is larger than the student, this same-architecture reference model supplies only the weight-space spectral target; the larger teacher still supplies $\Lkd$. This separates functional supervision from the student-space geometry required for spectral alignment.

Output-level KD constrains predictions but not the rank-$r_\ell$ adapter subspace; SAD-LoRA derives this subspace from the teacher update under the downstream activation distribution. For layer input $x^{(\ell)}$ with covariance $\Sigma_x^{(\ell)} = \mathbb{E}[x^{(\ell)}x^{(\ell)\top}]$, the layerwise surrogate error satisfies
\begin{equation}
\begin{aligned}
    \mathcal{E}^{(\ell)}
    &= \mathbb{E}\!\left[\big\|(\DWT^{(\ell)} - B^{(\ell)}A^{(\ell)})\,x^{(\ell)}\big\|_2^2\right] \\
    &= \big\|(\DWT^{(\ell)} - B^{(\ell)}A^{(\ell)})\,(\Sigma_x^{(\ell)})^{1/2}\big\|_F^2 \\
    &= \big\|\tWT^{(\ell)} - B^{(\ell)}A^{(\ell)}(\Sigma_x^{(\ell)})^{1/2}\big\|_F^2 ,
\end{aligned}
\label{eq:data_weighted_error}
\end{equation}
where the \emph{data-weighted teacher update} is
\begin{equation}
    \tWT^{(\ell)} = \DWT^{(\ell)} (\Sigma_x^{(\ell)})^{1/2} .
    \label{eq:data_weighted_update}
\end{equation}
The relevant spectral object is therefore $\tWT^{(\ell)}$, not the unweighted update alone. Let its compact SVD be
\begin{equation}
    \tWT^{(\ell)} = \Ut^{(\ell)}\widetilde{\Sigma}_T^{(\ell)} \widetilde{V}_T^{(\ell)\top},
    \qquad \sigt_1^{(\ell)} \geq \sigt_2^{(\ell)} \geq \cdots .
    \label{eq:data_weighted_svd}
\end{equation}
The leading left singular vectors $\Utr^{(\ell)} := \Ut^{(\ell)}[:,1{:}r_\ell]$ define the target output-side subspace for $B^{(\ell)}$.

\textbf{Implicit SVD.} The covariance is never formed explicitly. Given centered calibration activations $X_c^{(\ell)} \in \mathbb{R}^{n\times \din}$ collected from the student backbone on a small calibration set, define
\begin{equation}
    Z^{(\ell)} = \tfrac{1}{\sqrt{n-1}}\,\DWT^{(\ell)} X_c^{(\ell)\top}
    \in \mathbb{R}^{\dout\times n}.
    \label{eq:implicit_svd}
\end{equation}
Because $\widehat{\Sigma}_x^{(\ell)} = \tfrac{1}{n-1} X_c^{(\ell)\top} X_c^{(\ell)}$, we have
\begin{equation}
    Z^{(\ell)} Z^{(\ell)\top}
    = \DWT^{(\ell)}\widehat{\Sigma}_x^{(\ell)}\DWT^{(\ell)\top}
    = \tWT^{(\ell)}\tWT^{(\ell)\top},
\end{equation}
so the left singular vectors of $Z^{(\ell)}$ are the empirical $\Utr^{(\ell)}$ and its singular values are the empirical $\sigt_i^{(\ell)}$. This avoids forming the $\din\times \din$ covariance matrix when $n \ll \din$. The targets $\Utr^{(\ell)}$ and $\sigt_i^{(\ell)}$ are computed once and frozen for the rest of training.

\subsection{Spectral Alignment and Coefficient Losses}
\label{sec:spectral_alignment_distillation}

The matrix $B^{(\ell)}$ alone determines the output-side subspace accessible to the LoRA adapter. Let $B^{(\ell)} = Q_B^{(\ell)} R_B^{(\ell)}$ be its thin QR decomposition. SAD-LoRA penalizes the principal-angle discrepancy between $\colspan(B^{(\ell)})$ and $\colspan(\Utr^{(\ell)})$:
\begin{equation}
\begin{aligned}
    \Lalign^{(\ell)}
    &= 1 - \tfrac{1}{r_\ell}\big\| Q_B^{(\ell)\top}\,\Utr^{(\ell)} \big\|_F^2 \\
    &= \tfrac{1}{r_\ell}\sum_{i=1}^{r_\ell} \sin^2\theta_i^{(\ell)} ,
\end{aligned}
\label{eq:alignment_loss}
\end{equation}
where $\theta_i^{(\ell)}$ are the principal angles between the two $r_\ell$-dimensional subspaces. The loss is zero exactly when the adapter and target subspaces coincide and one when they are orthogonal.

Because $Q_B^{(\ell)}$ is orthonormal, the nonzero singular values of $B^{(\ell)}A^{(\ell)}$ equal those of $R_B^{(\ell)}A^{(\ell)}$, which is an $r_\ell\times \din$ matrix and therefore much cheaper to decompose. We define a spectral coefficient-matching loss
\begin{equation}
    \Lcoeff^{(\ell)}
    = \frac{\big\|\sigma\!\left(R_B^{(\ell)}A^{(\ell)}\right) - \sigt_{T,r}^{(\ell)}\big\|_2^2}{\big\|\sigt_{T,r}^{(\ell)}\big\|_2^2 + \varepsilon} ,
    \label{eq:coefficient_loss}
\end{equation}
where $\sigt_{T,r}^{(\ell)}=(\sigt_1^{(\ell)},\ldots,\sigt_{r_\ell}^{(\ell)})$ and $\varepsilon$ is a small constant. This is an unweighted proxy for the data-weighted coefficient objective analyzed in Section~\ref{sec:analysis}; it avoids applying $\Sigma_x^{1/2}$ at every step and is used only as an auxiliary regularizer.

\subsection{Training Objective}
\label{sec:training_objective}

For classification, we use temperature-scaled logit distillation; for regression, we use mean-squared error between teacher and student predictions. Let $z_T,z_S$ denote teacher and student logits and $\tau$ the distillation temperature:
\begin{equation}
    \Lkd = \tau^2\,\mathrm{KL}\!\left(\mathrm{softmax}(z_T/\tau)\,\|\,\mathrm{softmax}(z_S/\tau)\right).
    \label{eq:kd_loss}
\end{equation}
With supervised task loss $\mathcal{L}_{\mathrm{task}}$, the SAD-LoRA objective is
\begin{equation}
\begin{aligned}
    \mathcal{L}
    &= \mathcal{L}_{\mathrm{task}} + \lambda_{\mathrm{KD}}\,\Lkd \\
    &\quad+ \alpha\,\frac{1}{|\mathcal{S}|}\sum_{\ell\in\mathcal{S}}\Lalign^{(\ell)}
    + \beta\,\frac{1}{|\mathcal{S}|}\sum_{\ell\in\mathcal{S}}\Lcoeff^{(\ell)} ,
\end{aligned}
\label{eq:sadlora_objective}
\end{equation}
where $\mathcal{S}$ is the set of adapted layers. The full objective uses both spectral terms ($\alpha,\beta>0$); the alignment-only variant sets $\beta=0$, the coefficient-only variant sets $\alpha=0$. We keep the spectral weights $\alpha$ and $\beta$ fixed throughout training; the optimizer uses the learning-rate warmup and decay described in Appendix~\ref{app:experimental_details}.

\subsection{Spectral Initialization}
\label{sec:initialization_optimization}

SAD-LoRA initializes the adapter directly in the target subspace:
\begin{equation}
    B^{(\ell)} = \Utr^{(\ell)},\qquad A^{(\ell)} = 0 .
    \label{eq:spectral_initialization}
\end{equation}
Two consequences follow. First, $B^{(\ell)}A^{(\ell)} = 0$ at $t{=}0$, so the student function exactly reproduces the pretrained backbone. Second, $\Lalign^{(\ell)}$ is identically zero at initialization, so $\Lkd$ supplies the dominant gradient signal from the first step and $\Lalign$ acts as a regularizer that prevents drift rather than fighting the data fit.
If the requested rank exceeds the effective target rank determined from the data-weighted spectral tail (see Proposition~\ref{prop:rank_suff}), we fill the remaining columns of $B^{(\ell)}$ with a random orthonormal complement obtained by Gram-Schmidt against $\Utr^{(\ell)}$. Zero-padding would leave $B^{(\ell)}$ rank-deficient and lock $\Lalign$ above a hard floor regardless of training; orthonormal-complement padding guarantees both $\Lalign^{(0)} = 0$ and full column rank.

\subsection{Training Procedure and Cost}
\label{sec:algorithm}

Spectral targets are computed once offline by forming the sketch $Z^{(\ell)}$ in Eq.~\ref{eq:implicit_svd} from calibration activations and registering $\Utr^{(\ell)}$ and $\sigt_{T,r}^{(\ell)}$ as frozen buffers. Training then initializes $B^{(\ell)}=\Utr^{(\ell)}$, sets $A^{(\ell)}=0$, and updates only the LoRA parameters using Eq.~\ref{eq:sadlora_objective}. 
The additional per-step cost is one thin QR of $B^{(\ell)}$ for $\Lalign$ and, when $\Lcoeff$ is used, one SVD of $R_B^{(\ell)}A^{(\ell)}\in\mathbb{R}^{r\times\din}$, giving $O(r^2(\dout+\din))$ per adapted layer. The targets are fixed after preprocessing, the backbone remains frozen, and no inference-time module is added. Full pseudocode is given in Appendix~\ref{app:algorithm}.

\section{Analysis}
\label{sec:analysis}

We analyze a single adapted layer and suppress the layer index. The data-weighted distillation error is
\begin{equation}
    \mathcal{E}(B,A) = \big\|\tWT - BA\Sigma_x^{1/2}\big\|_F^2 ,
    \label{eq:analysis_error}
\end{equation}
with $\tWT = \DWT\Sigma_x^{1/2} = \Ut\widetilde\Sigma_T \widetilde V_T^\top$ its compact SVD.

\subsection{Projector-Based Error Decomposition}

Our first result is an exact projector-based decomposition of the data-weighted distillation error.

\begin{theorem}[Distillation error decomposition]
\label{thm:decomp}
Let $B \in \mathbb{R}^{\dout\times r}$ have full column rank and let $P_B = B(B^\top B)^{-1}B^\top$ be the orthogonal projector onto $\colspan(B)$. Let $\Utr$ contain the top $r$ left singular vectors of $\tWT$ and let $\mathcal{R} := \sum_{i>r}\sigt_i^2$ be the rank-$r$ tail energy. Then
\begin{equation}
    \mathcal{E}(B,A) \;=\; \mathcal{S}(B) \;+\; \mathcal{C}(B,A) \;+\; \mathcal{R} ,
    \label{eq:three_term}
\end{equation}
where
\begin{align}
    \mathcal{S}(B) &:= \big\|(I-P_B)\tWT\big\|_F^2 \;-\; \mathcal{R} \;\geq\; 0, \\
    \mathcal{C}(B,A) &:= \big\|P_B\tWT - BA\Sigma_x^{1/2}\big\|_F^2 \;\geq\; 0 .
\end{align}
The terms are interpretable as: \emph{(I) $\mathcal{S}(B)$ — subspace misalignment}, depending only on $\colspan(B)$ and irreducible without rotating $B$; \emph{(II) $\mathcal{C}(B,A)$ — within-subspace coefficient mismatch}, eliminable by choosing $A$ optimally given $B$; and \emph{(III) $\mathcal{R}$ — rank residual}, fixed by the rank budget. Furthermore, if $\sigt_r>\sigt_{r+1}$, then $\mathcal{S}(B)=0$ if and only if $\colspan(B)=\colspan(\Utr)$. Under boundary degeneracy, $\mathcal{S}(B)=0$ for any optimal rank-$r$ left singular subspace within the degenerate eigenspace; SAD-LoRA uses the empirical SVD target $\Utr$ as the alignment reference.
\end{theorem}

\begin{proof}[Sketch]
Since $\colspan(BA\Sigma_x^{1/2}) \subseteq \colspan(B)$, Pythagoras gives
\[
\mathcal{E}(B,A)
=
\|(I-P_B)\tWT\|_F^2
+
\|P_B\tWT - BA\Sigma_x^{1/2}\|_F^2 .
\]
By Eckart--Young--Mirsky, the first term is lower-bounded by $\mathcal{R}$, with equality at a rank-$r$ optimal left singular subspace, up to non-uniqueness under singular-value degeneracy. This yields Eq.~\ref{eq:three_term}; the second term is nonnegative by definition. Full proof in Appendix~\ref{app:proof_thm1}.
\end{proof}

Theorem~\ref{thm:decomp} clarifies why output-level KD can waste rank budget. Standard KD updates $B$ and $A$ through output gradients and can reduce coefficient error inside a misaligned subspace while leaving $\mathcal{S}$ large. SAD-LoRA addresses this by adding $\Lalign$ (Section~\ref{sec:spectral_alignment_distillation}), which explicitly targets the subspace term.

\textbf{$\Lalign$ as a Grassmannian surrogate for $\mathcal{S}$.}
The alignment loss
\begin{equation}
    \Lalign \;=\; \tfrac{1}{r}\!\sum_{i=1}^{r}\sin^2\theta_i
    \label{eq:alignment_grassmann}
\end{equation}
is the normalized squared chordal distance on $\mathrm{Gr}(r,\dout)$ between $\colspan(B)$ and the selected target subspace $\colspan(\Utr)$. When the top-$r$ subspace is unique, $\Lalign=0$ and $\mathcal{S}(B)=0$ have the same zero set. Under degeneracy, $\mathcal{S}$ may vanish for multiple optimal subspaces, while $\Lalign$ anchors the adapter to the empirical SVD target used by SAD-LoRA. Thus $\Lalign$ is not equal to $\mathcal{S}$; it is an unweighted, differentiable surrogate that is cheaper and more stable than estimating spectrally weighted angles directly.

\textbf{Relation to $\Lcoeff$.} The principled, theory-aligned coefficient loss is the data-weighted form
\begin{equation}
    \Lcoeff^{\mathrm{dw}} = \frac{\big\|\sigma\!\left(R_B A \Sigma_x^{1/2}\right) - \sigt_{T,r}\big\|_2^2}{\big\|\sigt_{T,r}\big\|_2^2 + \varepsilon} ,
    \label{eq:Lcoeff_dw}
\end{equation}
which corresponds to matching the within-subspace data-weighted spectrum implied by $\mathcal{C}$. Computing it requires applying $\Sigma_x^{1/2}$ (or equivalently right-multiplying by the calibration sketch $X_c$) at every step, which doubles the per-step cost of the coefficient term. Because our ablations show $\Lcoeff$ to be auxiliary rather than load-bearing, we adopt the cheaper unweighted proxy $\Lcoeff = \|\sigma(R_B A) - \sigt_{T,r}\|_2^2 / (\|\sigt_{T,r}\|_2^2 + \varepsilon)$ in our main experiments. The proxy equals $\Lcoeff^{\mathrm{dw}}$ when $\Sigma_x=I$; in the anisotropic case it is only an approximation to the data-weighted coefficient objective. We report results with the unweighted form throughout; preliminary runs with $\Lcoeff^{\mathrm{dw}}$ show no consistent advantage on GLUE.

\subsection{Optimal Subspace and Rank Sufficiency}

\begin{proposition}[Data-weighted teacher subspace]
\label{prop:data_weighted_subspace}
Among all rank-$r$ matrices $M$, an optimal minimizer of $\|\tWT - M\|_F^2$ has column space $\colspan(\Utr)$ and residual $\sum_{i>r}\sigt_i^2$. If $\sigt_r>\sigt_{r+1}$, this optimal column space is unique.
\end{proposition}

\begin{proof}
Direct from Eckart--Young--Mirsky; see Appendix~\ref{app:proof_prop1}.
\end{proof}

The relevant low-rank target is therefore not the unweighted teacher update, but the teacher update composed with $\Sigma_x^{1/2}$. Combined with Theorem~\ref{thm:decomp}, this identifies $\colspan(\Utr)$ as an optimal minimizer of $\mathcal{S}(B)$, unique when $\sigt_r>\sigt_{r+1}$, and justifies $\Lalign$ as its Grassmannian surrogate.

\begin{proposition}[Rank sufficiency]
\label{prop:rank_suff}
For any $\epsilon>0$, the minimum rank required to drive $\mathcal{E}(B,A) < \epsilon$ over all $B,A$ is
\begin{equation}
    r^{*}_\epsilon(\DWT,\Sigma_x) = \min\!\left\{k : \sum_{i>k}\sigt_i^2 < \epsilon\right\} .
    \label{eq:rank_suff}
\end{equation}
\end{proposition}

\begin{proof}
At rank $k$, Eckart--Young--Mirsky gives the minimum attainable error $\sum_{i>k}\sigt_i^2$. The stated criterion selects the smallest $k$ for which this residual is below $\epsilon$. See Appendix~\ref{app:proof_prop2}.
\end{proof}

This is a layerwise, data-dependent rank criterion: layers with sharp spectral decay need fewer ranks, while layers with shallow decay require more. In the main experiments, we use a uniform rank $r$ across layers to match prior LoRA baselines; Eq.~\ref{eq:rank_suff} serves as a diagnostic and motivates future rank-allocation variants.

\textbf{Implication for intruder dimensions.}
Any component of $B$ outside the selected target subspace is penalized by $\Lalign$. Under $\alpha>0$, the optimizer must trade off any output-level gain from such directions against the geometric penalty, providing an explicit mechanism against intruder dimensions \citep{shuttleworth2024lora}.

\section{Experiments}
\label{sec:experiments}

We evaluate SAD-LoRA in two settings: controlled synthetic experiments that directly test the decomposition in Theorem~\ref{thm:decomp}, and RoBERTa-based GLUE distillation across ranks $r\in\{1,2,4,8,16\}$. We then ablate alignment and coefficient matching to identify which spectral component drives the gains.

\subsection{Experimental Setup}
\label{sec:experimental_setup}

\textbf{Tasks and metrics.} We evaluate on six GLUE tasks \citep{wang2018glue}: SST-2, MRPC, STS-B, CoLA, QNLI, and RTE. We report accuracy for SST-2, QNLI, and RTE; F1 for MRPC; Spearman correlation for STS-B; and Matthews correlation (MCC) for CoLA. All GLUE results are mean and standard deviation over three seeds.

\textbf{Models and adapted modules.} The student is RoBERTa-base \citep{liu2020roberta} with LoRA adapters on the query and value projections in each Transformer block. Functional distillation logits come from a RoBERTa-large teacher. Because SAD-LoRA requires teacher updates in the same parameter space as the student, spectral targets are computed from a same-architecture RoBERTa-base reference fine-tuned on the same task; the larger teacher continues to supply $\Lkd$. We evaluate $r\in\{1,2,4,8,16\}$.

\textbf{Compared methods.} All baselines are reproduced in the same training pipeline with matched optimizers and schedules: \emph{KD-LoRA} (KL distillation with LoRA \citep{azimi2024kdlora}), \emph{Logit-MSE} (replaces KL with MSE on logits), \emph{NN-Init} (Neural Nuggets \citep{zhong2024seeking}, SVD initialization from $\DWT$ without alignment objective), and \emph{PiSSA-Init} (principal-singular-vector initialization from $W_0$ \citep{meng2024pissa}). We additionally evaluate two SAD-LoRA ablations: \emph{SAD-LoRA-Align} ($\beta=0$, alignment only) and \emph{SAD-LoRA-Coeff} ($\alpha=0$, coefficient only). Hyperparameters are matched across methods; SAD-LoRA uses $\alpha=0.2$, $\beta=0.05$, $\tau=4$, and calibration size $n=1024$, with optimization details in Appendix~\ref{app:experimental_details}.

\subsection{Controlled Spectral Validation}
\label{sec:controlled_spectral_validation}

The synthetic experiments test whether Theorem~\ref{thm:decomp} predicts the actual error budget and whether $\Lalign$ corrects the predicted failure mode. We generate teacher updates with controlled spectral gap $\gamma=\sigt_r/\sigt_{r+1}$, train rank-$4$ adapters, and compute $\mathcal{S}$, $\mathcal{C}$, and $\mathcal{R}$ exactly from the trained $(B,A)$ and synthetic teacher.

\begin{figure}[t!]
    \centering
    \includegraphics[width=\columnwidth]{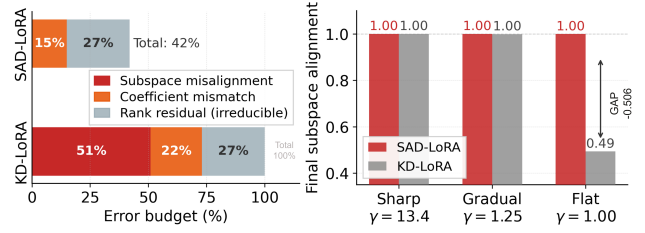}
    \caption{Controlled spectral validation. \emph{Left:} under a flat teacher spectrum, KD-LoRA spends most of its error on subspace misalignment $\mathcal{S}$ ($51\%$), whereas SAD-LoRA drives $\mathcal{S}$ nearly to zero, leaving the rank residual $\mathcal{R}$ ($27\%$) and coefficient mismatch $\mathcal{C}$ ($15\%$). \emph{Right:} final alignment $\tfrac{1}{r}\|Q_B^\top\Utr\|_F^2$ across spectral gaps. Both methods align under sharp spectra, but only SAD-LoRA remains aligned in the flat case ($\gamma=1$).}
    \label{fig:synthetic_validation}
\end{figure}

\begin{table}[t]
\centering
\caption{Synthetic three-term error decomposition at $r=4$ with flat teacher spectrum ($\gamma=1$) and identity covariance.}
\label{tab:synthetic_decomp}

\small
\setlength{\tabcolsep}{5pt}
\renewcommand{\arraystretch}{1.05}
\begin{tabular}{@{}lcccc@{}}
\toprule
Method &
\shortstack{$\mathcal{S}$\\subspace} &
\shortstack{$\mathcal{C}$\\coeff.} &
\shortstack{$\mathcal{R}$\\residual} &
Total \\
\midrule
KD-LoRA   & $51\%$ & $22\%$ & $27\%$ & $100\%$ \\
SAD-LoRA & $\mathbf{\approx 0\%}$ & $15\%$ & $27\%$ & $\mathbf{42\%}$ \\
\bottomrule
\end{tabular}
\end{table}

\begin{table*}[ht!]
\centering
\caption{
GLUE results at ranks $r=4$ and $r=8$, reported as mean $\pm$ std over three seeds.
SST-2, QNLI, and RTE use accuracy; MRPC uses F1; STS-B uses Spearman's $\rho$; and CoLA uses Matthews correlation. Best results/near-ties under seed variability are also bolded.
}
\label{tab:glue_r4_r8}

\vspace{-0.5em}
\footnotesize
\setlength{\tabcolsep}{3.2pt}
\renewcommand{\arraystretch}{0.90}

\begin{tabular*}{\textwidth}{@{\extracolsep{\fill}}lcccccc@{}}
\toprule
Method & SST-2 & MRPC & STS-B & CoLA & QNLI & RTE \\
\midrule

\multicolumn{7}{@{}l}{\textit{Rank $r=4$}} \\
\midrule
KD-LoRA
& $93.3\!\pm\!0.4$
& $87.8\!\pm\!0.9$
& $0.789\!\pm\!0.033$
& $0.426\!\pm\!0.013$
& $91.3\!\pm\!0.2$
& $65.0\!\pm\!1.6$ \\

Logit-MSE
& $93.8\!\pm\!0.2$
& $87.2\!\pm\!0.2$
& $0.789\!\pm\!0.033$
& $0.423\!\pm\!0.013$
& $90.8\!\pm\!0.6$
& $59.0\!\pm\!3.7$ \\

NN-Init
& $93.4\!\pm\!0.4$
& $\mathbf{90.3\!\pm\!0.6}$
& $\mathbf{0.887\!\pm\!0.003}$
& $0.555\!\pm\!0.008$
& $\mathbf{92.5\!\pm\!0.1}$
& $69.2\!\pm\!1.0$ \\

PiSSA-Init
& $93.7\!\pm\!0.8$
& $90.2\!\pm\!0.6$
& $0.880\!\pm\!0.004$
& $0.519\!\pm\!0.015$
& $92.1\!\pm\!0.1$
& $66.1\!\pm\!1.2$ \\

\addlinespace[0.15em]
SAD-LoRA
& $94.3\!\pm\!0.2$
& $89.9\!\pm\!0.3$
& $0.884\!\pm\!0.002$
& $\mathbf{0.558\!\pm\!0.006}$
& $92.1\!\pm\!0.0$
& $69.9\!\pm\!1.3$ \\

SAD-LoRA-Align
& $94.1\!\pm\!0.1$
& $90.1\!\pm\!0.2$
& $\mathbf{0.887\!\pm\!0.003}$
& $0.553\!\pm\!0.003$
& $\mathbf{92.5\!\pm\!0.1}$
& $\mathbf{70.0\!\pm\!1.3}$ \\

SAD-LoRA-Coeff
& $\mathbf{94.5\!\pm\!0.1}$
& $90.0\!\pm\!0.4$
& $0.885\!\pm\!0.002$
& $0.539\!\pm\!0.009$
& $92.2\!\pm\!0.0$
& $63.1\!\pm\!6.8$ \\

\midrule
\multicolumn{7}{@{}l}{\textit{Rank $r=8$}} \\
\midrule
KD-LoRA
& $93.7\!\pm\!0.2$
& $90.2\!\pm\!0.7$
& $0.847\!\pm\!0.018$
& $0.478\!\pm\!0.016$
& $92.0\!\pm\!0.2$
& $67.4\!\pm\!1.1$ \\

Logit-MSE
& $93.5\!\pm\!0.2$
& $89.8\!\pm\!0.2$
& $0.847\!\pm\!0.018$
& $0.473\!\pm\!0.019$
& $91.9\!\pm\!0.3$
& $65.2\!\pm\!2.8$ \\

NN-Init
& $93.8\!\pm\!0.1$
& $90.2\!\pm\!0.9$
& $0.892\!\pm\!0.004$
& $0.554\!\pm\!0.005$
& $\mathbf{92.8\!\pm\!0.2}$
& $72.1\!\pm\!1.4$ \\

PiSSA-Init
& $94.0\!\pm\!0.5$
& $\mathbf{91.1\!\pm\!0.6}$
& $0.887\!\pm\!0.006$
& $0.538\!\pm\!0.015$
& $92.3\!\pm\!0.1$
& $68.8\!\pm\!2.1$ \\

\addlinespace[0.15em]
SAD-LoRA
& $\mathbf{94.5\!\pm\!0.2}$
& $90.9\!\pm\!0.4$
& $0.891\!\pm\!0.003$
& $0.558\!\pm\!0.005$
& $92.5\!\pm\!0.1$
& $71.8\!\pm\!0.8$ \\

SAD-LoRA-Align
& $93.6\!\pm\!0.2$
& $90.1\!\pm\!0.3$
& $\mathbf{0.893\!\pm\!0.003}$
& $\mathbf{0.562\!\pm\!0.006}$
& $\mathbf{92.8\!\pm\!0.1}$
& $\mathbf{72.2\!\pm\!0.6}$ \\

SAD-LoRA-Coeff
& $93.7\!\pm\!0.2$
& $90.5\!\pm\!0.2$
& $0.891\!\pm\!0.003$
& $0.537\!\pm\!0.005$
& $92.5\!\pm\!0.1$
& $68.0\!\pm\!0.6$ \\

\bottomrule
\end{tabular*}

\vspace{-0.3cm}
\end{table*}

\textbf{Decomposition (Table~\ref{tab:synthetic_decomp}).}
The flat-spectrum case is difficult for output-level KD because the teacher update provides weak directional preference. KD-LoRA spends most of its error budget on subspace misalignment ($\mathcal{S}=51\%$), whereas SAD-LoRA drives $\mathcal{S}$ to nearly zero while leaving the irreducible rank residual $\mathcal{R}$ unchanged. The total error drops from $100\%$ to $42\%$.

\textbf{Alignment across spectral gaps.} Fig.~\ref{fig:synthetic_validation} (right) reports final subspace alignment $\tfrac{1}{r}\|Q_B^\top\Utr\|_F^2$ for sharp ($\gamma=13.4$), gradual ($\gamma=1.25$), and flat ($\gamma=1$) teacher spectra. Both methods recover the dominant subspace under sharp spectra ($1.00$ each). As the spectral gap shrinks, KD-LoRA degrades sharply: its final alignment drops to $0.49$ in the flat case, while SAD-LoRA remains at $1.00$. This supports the role of $\Lalign$: it is most useful when the teacher spectrum provides weak directional preference.

\subsection{Main GLUE Results}
\label{sec:main_glue_results}

We report exact results at $r=4$ and $r=8$, and show rank trajectories for $r\in\{1,2,4,8,16\}$.

\textbf{Rank-$4/8$ results (Table~\ref{tab:glue_r4_r8}).}
At $r=4$, SAD-LoRA-family variants are strongest or competitive on five of six tasks. SAD-LoRA-Coeff gives the best SST-2 score ($94.5$), full SAD-LoRA gives the best CoLA score ($0.558$), and SAD-LoRA-Align matches NN-Init on STS-B ($0.887$) and QNLI ($92.5$) while giving the best RTE score ($70.0$). MRPC is the exception: NN-Init remains strongest ($90.3$), with SAD-LoRA variants within $0.4$ F1. Relative to KD-LoRA, full SAD-LoRA improves STS-B from $0.789$ to $0.884$, CoLA from $0.426$ to $0.558$, and RTE from $65.0$ to $69.9$. Notably, SAD-LoRA at $r=4$ already exceeds KD-LoRA at $r=8$ on STS-B and CoLA, indicating that subspace selection can compensate for additional rank.
At $r=8$, margins shrink on saturated tasks, but the same pattern remains. Full SAD-LoRA gives the best SST-2 score ($94.5$), while SAD-LoRA-Align is strongest on STS-B ($0.893$), CoLA ($0.562$), and RTE ($72.2$), and ties NN-Init on QNLI ($92.8$). PiSSA-Init remains best on MRPC ($91.1$), with full SAD-LoRA close behind ($90.9$). The largest gains over KD-LoRA again occur on STS-B and CoLA, supporting the hypothesis that data-weighted teacher-update directions matter most when the adapter must choose a small task-relevant subspace.

\begin{figure*}[!t]
    \centering
    \includegraphics[width=0.9\textwidth]{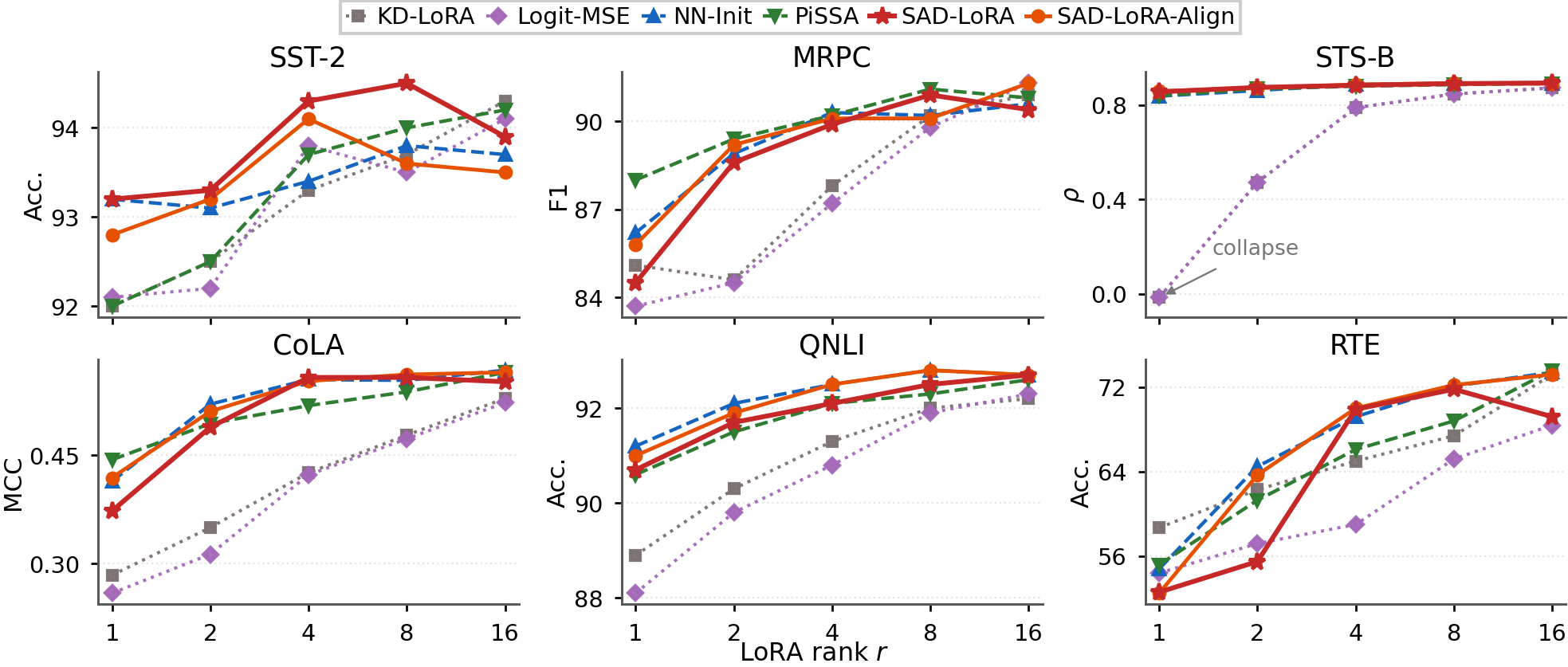}
    \caption{Rank efficiency on GLUE. Mean validation performance over three seeds as the LoRA rank varies from $1$ to $16$. SAD-LoRA variants are most effective in low-rank regimes. STS-B and CoLA show the clearest separation; SST-2 and QNLI are near-saturated; on MRPC, PiSSA-Init remains strongest at moderate rank. Logit-MSE collapses on STS-B at $r=1$ (Spearman near zero), highlighting the value of subspace-aware initialization for rank-limited correlation tasks.}
    \label{fig:rank_curves}
    \vspace{-0.4cm}
\end{figure*}

\begin{figure}[!t]
    \centering
    \includegraphics[width=0.75\linewidth]{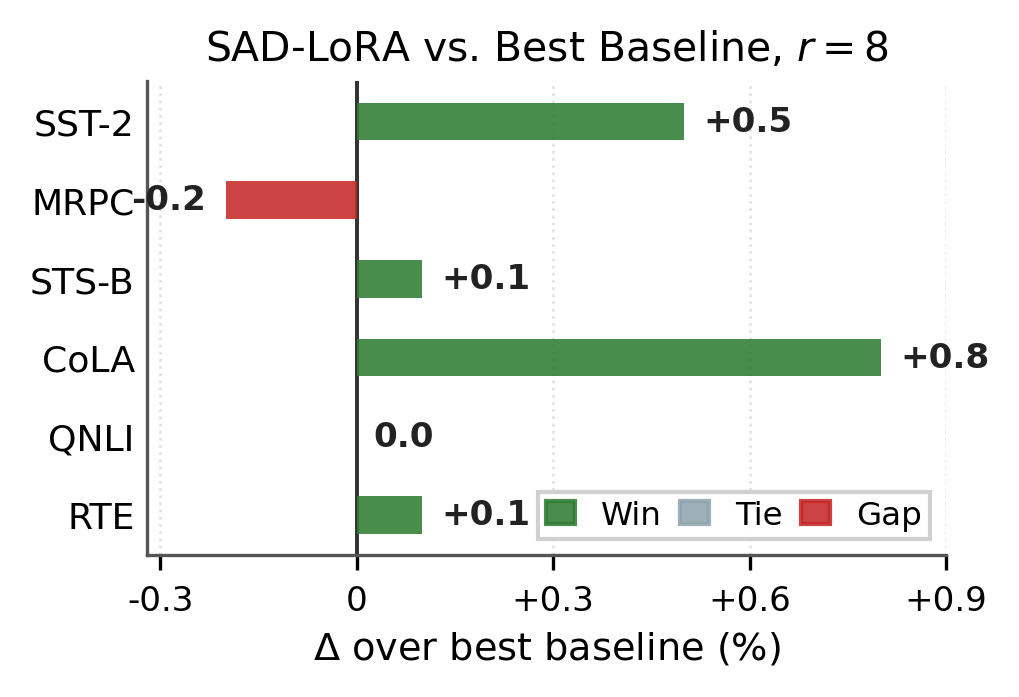}
    \caption{Best SAD-LoRA-family result minus best non-SAD baseline at $r=8$. Positive values indicate an advantage for SAD-LoRA. Correlation and MCC differences are shown in percentage-point units.}
    \label{fig:wins_r8}
    \vspace{-0.4cm}
\end{figure}

\textbf{Rank trajectories (Fig.~\ref{fig:rank_curves}).}
STS-B gives the clearest rank-efficiency signal: at $r=1$, Logit-MSE collapses to $-0.016$ Spearman while SAD-LoRA already reaches $0.85$, and by $r=4$ SAD-LoRA-Align and NN-Init match the best-rank KD baselines. CoLA shows a similar low-rank gain, whereas SST-2 and QNLI are nearly saturated and MRPC favors PiSSA-Init at moderate rank.

\subsection{Ablations and Rank Efficiency}
\label{sec:ablation_rank_analysis}

The ablations isolate three effects: output-level KD (KD-LoRA, Logit-MSE), spectral initialization without training-time alignment (NN-Init, PiSSA-Init), and SAD-LoRA's two spectral terms separately.

\textbf{Subspace alignment is load-bearing.} Removing $\Lalign$ (SAD-LoRA-Coeff) is more damaging than removing $\Lcoeff$. At $r=8$, dropping $\Lalign$ costs $0.025$ MCC on CoLA ($0.562\to0.537$) and $4.2$ accuracy points on RTE ($72.2\to68.0$). The variance of SAD-LoRA-Coeff is also larger on small-data tasks ($\pm 6.8$ on RTE at $r=4$). This is consistent with Theorem~\ref{thm:decomp}: $\Lcoeff$ targets $\mathcal{C}$, which is bounded by the within-subspace projection of $\tWT$ and is therefore meaningful only after $\mathcal{S}$ has been driven small.

\textbf{Coefficient matching is auxiliary.} Removing $\Lcoeff$ improves or matches the full objective on STS-B, CoLA, QNLI, and RTE while the full objective is stronger on SST-2 and MRPC. This suggests that exact singular-value matching is more sensitive to weighting than directional alignment and is best used as a regularizer rather than the central mechanism.

\textbf{Per-task summary at $r=8$ (Fig.~\ref{fig:wins_r8}).} SAD-LoRA variants improve over the best non-SAD baseline on SST-2 ($+0.5$), STS-B ($+0.1$), CoLA ($+0.8$), and RTE ($+0.1$); tie on QNLI; and remain $0.2$ F1 below PiSSA-Init on MRPC.

\textbf{Synthesis.} The experiments support two conclusions. (i) The synthetic decomposition confirms Theorem~\ref{thm:decomp}: standard KD's distillation error is dominated by subspace misalignment under flat spectra, and SAD-LoRA removes this term. (ii) GLUE results show that this mechanism translates into downstream gains in low-rank, spectrally sensitive regimes, especially on STS-B and CoLA.

\section{Discussion and Limitations}
\label{sec:discussion}

SAD-LoRA's gains come primarily from controlling the adapter subspace during training, not merely from spectral initialization. This distinction is important because PiSSA-Init and NN-Init already use SVD-based initialization, yet they do not keep $\colspan(B)$ tied to the data-weighted teacher-update subspace. The synthetic decomposition shows that this geometric target removes the subspace-misalignment term that KD-LoRA does not explicitly penalize, while the GLUE results show corresponding rank-efficiency gains: at $r=4$, SAD-LoRA already matches or exceeds KD-LoRA at $r=8$ on STS-B and CoLA. The ablations further indicate that $\Lalign$ is the central mechanism. Coefficient matching is useful as an auxiliary regularizer---especially on tasks such as SST-2 and MRPC---but it is less robust than directional alignment and should be treated as optional rather than load-bearing.

\textbf{Limitations.}
SAD-LoRA requires a teacher update in the same weight space as the student adapter. When the functional teacher is larger than the student, we therefore use a same-architecture fine-tuned reference model to construct the spectral targets, adding one offline model to the pipeline. The method also requires a calibration pass to compute layerwise targets, although this cost is incurred once and adds no inference overhead. Empirically, the gains are not uniform: MRPC remains favorable to PiSSA-Init at moderate rank, suggesting that pretrained-weight principal structure can be preferable for some semantic-matching tasks. Finally, our experiments focus on encoder-model distillation on GLUE; extending teacher-update alignment to decoder-only LLMs and generation tasks remains future work.


\section{Conclusion}
\label{sec:conclusion}

We introduced SAD-LoRA, a spectral alignment framework for low-rank knowledge distillation. SAD-LoRA derives a layerwise target subspace from the data-weighted student-space reference update $\DWT\Sigma_x^{1/2}$ and regularizes the LoRA adapter to remain aligned with this subspace during training. The analysis decomposes the data-weighted distillation error into subspace misalignment, coefficient mismatch, and rank residual, showing why output-level KD alone can waste rank budget in a misaligned adapter subspace. Synthetic experiments verify this mechanism directly, and RoBERTa-based GLUE distillation shows measurable rank-efficiency gains, especially on STS-B and CoLA. Ablations indicate that principal-angle subspace alignment is the central component, while coefficient matching is useful but auxiliary.

\section*{Acknowledgment}
This work was supported by the Institute of Information \& Communications Technology Planning \& Evaluation (IITP) grant funded by the Korea government (MSIT) (No.~RS-2024-00443391, Korea-Japan Joint International Research on Deep Learning-based Autonomous Mobility Control Technology for Autonomous Mobile Robots).

\bibliography{icml}

@misc{han2024peftsurvey,
  title        = {Parameter-Efficient Fine-Tuning for Large Models: A Comprehensive Survey},
  author       = {Han, Zeyu and Gao, Chao and Liu, Jinyang and Zhang, Jeff and Zhang, Sai Qian},
  year         = {2024},
  eprint       = {2403.14608},
  archivePrefix= {arXiv},
  primaryClass = {cs.LG},
  url          = {https://arxiv.org/abs/2403.14608}
}

@article{hu2022lora,
  title={Lora: Low-rank adaptation of large language models.},
  author={Hu, Edward J and Shen, Yelong and Wallis, Phillip and Allen-Zhu, Zeyuan and Li, Yuanzhi and Wang, Shean and Wang, Liang and Chen, Weizhu and others},
  journal={Iclr},
  volume={1},
  number={2},
  pages={3},
  year={2022}
}

@article{shuttleworth2024lora,
  title={Lora vs full fine-tuning: An illusion of equivalence},
  author={Shuttleworth, Reece and Andreas, Jacob and Torralba, Antonio and Sharma, Pratyusha},
  journal={arXiv preprint arXiv:2410.21228},
  year={2024}
}

@article{soarez2026demystifying,
  title   = {Demystifying Low-Rank Knowledge Distillation in Large Language Models: Convergence, Generalization, and Information-Theoretic Guarantees},
  author  = {Soarez, Alberlucia Rafael and Kim, Daniel and Costa, Mariana and Torre, Alejandro},
  journal = {arXiv preprint arXiv:2603.22355},
  year    = {2026},
  doi     = {10.48550/arXiv.2603.22355},
}

@inproceedings{houlsby2019adapter,
  title = {Parameter-Efficient Transfer Learning for {NLP}},
  author = {Houlsby, Neil and Giurgiu, Andrei and Jastrzebski, Stanislaw and Morrone, Bruna and De Laroussilhe, Quentin and Gesmundo, Andrea and Attariyan, Mona and Gelly, Sylvain},
  booktitle = {Proceedings of the 36th International Conference on Machine Learning},
  pages = {2790--2799},
  year = {2019},
  volume = {97},
  series = {Proceedings of Machine Learning Research},
  publisher = {PMLR},

}

@inproceedings{li2021prefix,
  title = {Prefix-Tuning: Optimizing Continuous Prompts for Generation},
  author = {Li, Xiang Lisa and Liang, Percy},
  booktitle = {Proceedings of the 59th Annual Meeting of the Association for Computational Linguistics and the 11th International Joint Conference on Natural Language Processing},
  pages = {4582--4597},
  year = {2021},
  address = {Online},
  publisher = {Association for Computational Linguistics},
  doi = {10.18653/v1/2021.acl-long.353},

}

@inproceedings{lester2021prompt,
  title = {The Power of Scale for Parameter-Efficient Prompt Tuning},
  author = {Lester, Brian and Al-Rfou, Rami and Constant, Noah},
  booktitle = {Proceedings of the 2021 Conference on Empirical Methods in Natural Language Processing},
  pages = {3045--3059},
  year = {2021},
  address = {Online and Punta Cana, Dominican Republic},
  publisher = {Association for Computational Linguistics},
  doi = {10.18653/v1/2021.emnlp-main.243},

}

@inproceedings{benzaken2021bitfit,
  title = {{B}it{F}it: Simple Parameter-efficient Fine-tuning for Transformer-based Masked Language-models},
  author = {Ben Zaken, Elad and Goldberg, Yoav and Ravfogel, Shauli},
  booktitle = {Proceedings of the 60th Annual Meeting of the Association for Computational Linguistics},
  pages = {1--9},
  year = {2022},
  address = {Dublin, Ireland},
  publisher = {Association for Computational Linguistics},
  doi = {10.18653/v1/2022.acl-short.1},

}

@article{ding2023peftsurvey,
  title = {Parameter-efficient fine-tuning of large-scale pre-trained language models},
  author = {Ding, Ning and Qin, Yujia and Yang, Guang and Wei, Fuchao and Yang, Zonghan and Su, Yusheng and Hu, Shengding and Chen, Yulin and Chan, Chi-Min and Chen, Weize and Yi, Jing and Zhao, Weilin and Wang, Xiaozhi and Liu, Zhiyuan and Zheng, Hai-Tao and Chen, Jianfei and Liu, Yang and Tang, Jie and Li, Juanzi and Sun, Maosong},
  journal = {Nature Machine Intelligence},
  volume = {5},
  pages = {220--235},
  year = {2023},
  doi = {10.1038/s42256-023-00626-4},

}

@inproceedings{zhang2023adalora,
  title = {{AdaLoRA}: Adaptive Budget Allocation for Parameter-Efficient Fine-Tuning},
  author = {Zhang, Qingru and Chen, Minshuo and Bukharin, Alexander and Karampatziakis, Nikos and He, Pengcheng and Cheng, Yu and Chen, Weizhu and Zhao, Tuo},
  booktitle = {Proceedings of the 11th International Conference on Learning Representations},
  year = {2023},

}

@inproceedings{valipour2022dylora,
  title = {{DyLoRA}: Parameter-Efficient Tuning of Pre-trained Models using Dynamic Search-Free Low-Rank Adaptation},
  author = {Valipour, Mojtaba and Rezagholizadeh, Mehdi and Kobyzev, Ivan and Ghodsi, Ali},
  booktitle = {Proceedings of the 17th Conference of the European Chapter of the Association for Computational Linguistics},
  pages = {3274--3287},
  year = {2023},
  address = {Dubrovnik, Croatia},
  publisher = {Association for Computational Linguistics},
  doi = {10.18653/v1/2023.eacl-main.239},

}

@inproceedings{liu2024dora,
  title = {{DoRA}: Weight-Decomposed Low-Rank Adaptation},
  author = {Liu, Shih-Yang and Wang, Chien-Yi and Yin, Hongxu and Molchanov, Pavlo and Wang, Yu-Chiang Frank and Cheng, Kwang-Ting and Chen, Min-Hung},
  booktitle = {Proceedings of the 41st International Conference on Machine Learning},
  year = {2024},

}

@inproceedings{kopiczko2024vera,
  title = {{VeRA}: Vector-based Random Matrix Adaptation},
  author = {Kopiczko, Dawid J. and Blankevoort, Tijmen and Asano, Yuki M.},
  booktitle = {Proceedings of the 12th International Conference on Learning Representations},
  year = {2024},

}

@inproceedings{renduchintala2023tiedlora,
  title = {Tied-{LoRA}: Enhancing Parameter Efficiency of {LoRA} with Weight Tying},
  author = {Renduchintala, Adithya and Konuk, Tugrul and Kuchaiev, Oleksii},
  booktitle = {Proceedings of the 2024 Conference of the North American Chapter of the Association for Computational Linguistics: Human Language Technologies},
  pages = {8694--8705},
  year = {2024},
  publisher = {Association for Computational Linguistics},
  doi = {10.18653/v1/2024.naacl-long.481},

}

@inproceedings{li2024vblora,
  title = {{VB-LoRA}: Extreme Parameter Efficient Fine-Tuning with Vector Banks},
  author = {Li, Yang and Han, Shaobo and Ji, Shihao},
  booktitle = {Advances in Neural Information Processing Systems},
  volume = {37},
  year = {2024},

}

@inproceedings{dettmers2023qlora,
  title = {{QLoRA}: Efficient Finetuning of Quantized {LLMs}},
  author = {Dettmers, Tim and Pagnoni, Artidoro and Holtzman, Ari and Zettlemoyer, Luke},
  booktitle = {Advances in Neural Information Processing Systems},
  volume = {36},
  year = {2023},

}

@inproceedings{li2023loftq,
  title = {{LoftQ}: {LoRA}-Fine-Tuning-Aware Quantization for Large Language Models},
  author = {Li, Yixiao and Yu, Yifan and Liang, Chen and He, Pengcheng and Karampatziakis, Nikos and Chen, Weizhu and Zhao, Tuo},
  booktitle = {Proceedings of the 12th International Conference on Learning Representations},
  year = {2024},

}

@inproceedings{xu2023qalora,
  title = {{QA-LoRA}: Quantization-Aware Low-Rank Adaptation of Large Language Models},
  author = {Xu, Yuhui and Xie, Lingxi and Gu, Xiaotao and Chen, Xin and Chang, Heng and Zhang, Hengheng and Chen, Zhengsu and Zhang, Xiaopeng and Tian, Qi},
  booktitle = {Proceedings of the 12th International Conference on Learning Representations},
  year = {2024},

}

@inproceedings{meng2024pissa,
  title = {{PiSSA}: Principal Singular Values and Singular Vectors Adaptation of Large Language Models},
  author = {Meng, Fanxu and Wang, Zhaohui and Zhang, Muhan},
  booktitle = {Advances in Neural Information Processing Systems},
  volume = {37},
  year = {2024},

}

@inproceedings{wang-etal-2025-milora,
    title = "{M}i{L}o{RA}: Harnessing Minor Singular Components for Parameter-Efficient {LLM} Finetuning",
    author = "Wang, Hanqing  and
      Li, Yixia  and
      Wang, Shuo  and
      Chen, Guanhua  and
      Chen, Yun",
    editor = "Chiruzzo, Luis  and
      Ritter, Alan  and
      Wang, Lu",
    booktitle = "Proceedings of the 2025 Conference of the Nations of the Americas Chapter of the Association for Computational Linguistics: Human Language Technologies (Volume 1: Long Papers)",
    month = apr,
    year = "2025",
    address = "Albuquerque, New Mexico",
    publisher = "Association for Computational Linguistics",
    url = "https://aclanthology.org/2025.naacl-long.248/",
    doi = "10.18653/v1/2025.naacl-long.248",
    pages = "4823--4836",
    ISBN = "979-8-89176-189-6"
}

@misc{buehler2024olora,
  title = {{OLoRA}: Orthonormal Low-Rank Adaptation of Large Language Models},
  author = {B{\"u}y{\"u}kaky{\"u}z, Kerim},
  year = {2024},
  eprint = {2406.01775},
  archivePrefix = {arXiv},
  primaryClass = {cs.CL},
  doi = {10.48550/arXiv.2406.01775},

}

@inproceedings{yang2024corda,
  title = {{CorDA}: Context-Oriented Decomposition Adaptation of Large Language Models for Task-Aware Parameter-Efficient Fine-tuning},
  author = {Yang, Yibo and Li, Xiaojie and Zhou, Zhongzhu and Song, Shuaiwen Leon and Wu, Jianlong and Nie, Liqiang and Ghanem, Bernard},
  booktitle = {Advances in Neural Information Processing Systems},
  volume = {37},
  year = {2024},

}

@misc{yuan2023asvd,
  title = {{ASVD}: Activation-aware Singular Value Decomposition for Compressing Large Language Models},
  author = {Yuan, Zhihang and Shang, Yuzhang and Song, Yue and Yang, Dawei and Wu, Qiang and Yan, Yan and Sun, Guangyu},
  year = {2023},
  eprint = {2312.05821},
  archivePrefix = {arXiv},
  primaryClass = {cs.CL},
  doi = {10.48550/arXiv.2312.05821},

}

@inproceedings{lin2024awq,
  title = {{AWQ}: Activation-aware Weight Quantization for On-Device {LLM} Compression and Acceleration},
  author = {Lin, Ji and Tang, Jiaming and Tang, Haotian and Yang, Shang and Chen, Wei-Ming and Wang, Wei-Chen and Xiao, Guangxuan and Dang, Xingyu and Gan, Chuang and Han, Song},
  booktitle = {Proceedings of Machine Learning and Systems},
  volume = {6},
  year = {2024},

}

@misc{hinton2015distilling,
  title = {Distilling the Knowledge in a Neural Network},
  author = {Hinton, Geoffrey and Vinyals, Oriol and Dean, Jeff},
  year = {2015},
  eprint = {1503.02531},
  archivePrefix = {arXiv},
  primaryClass = {stat.ML},
  note = {{NIPS} Deep Learning and Representation Learning Workshop},

}

@article{gou2021knowledge,
  title = {Knowledge Distillation: A Survey},
  author = {Gou, Jianping and Yu, Baosheng and Maybank, Stephen J. and Tao, Dacheng},
  journal = {International Journal of Computer Vision},
  volume = {129},
  number = {6},
  pages = {1789--1819},
  year = {2021},
  publisher = {Springer},
  doi = {10.1007/s11263-021-01453-z},

}

@inproceedings{azimi2024kdlora,
  title = {{KD-LoRA}: A Hybrid Approach to Efficient Fine-Tuning with {LoRA} and Knowledge Distillation},
  author = {Azimi, Rambod and Rishav, Rishav and Teichmann, Marek and Ebrahimi Kahou, Samira},
  booktitle = {Proceedings of The 4th NeurIPS Efficient Natural Language and Speech Processing Workshop},
  pages = {73--80},
  year = {2024},
  editor = {Rezagholizadeh, Mehdi and Passban, Peyman and Samiee, Soheila and Partovi Nia, Vahid and Cheng, Yu and Deng, Yue and Liu, Qun and Chen, Boxing},
  volume = {262},
  series = {Proceedings of Machine Learning Research},
  publisher = {PMLR},

}

@misc{yang2024llmneo,
  title = {{LLM-Neo}: Parameter Efficient Knowledge Distillation for Large Language Models},
  author = {Yang, Runming and Wu, Taiqiang and Wang, Jiahao and Hu, Pengfei and Wu, Yik-Chung and Wong, Ngai and Yang, Yujiu},
  year = {2024},
  eprint = {2411.06839},
  archivePrefix = {arXiv},
  primaryClass = {cs.CL},
  doi = {10.48550/arXiv.2411.06839},

}

@inproceedings{hwang2024pclora,
  title = {{PC-LoRA}: Low-Rank Adaptation for Progressive Model Compression with Knowledge Distillation},
  author = {Hwang, Injoon and Park, Haewon and Lee, Youngwan and Yang, Jooyoung and Maeng, SunJae},
  booktitle = {Proceedings of the IEEE/CVF Conference on Computer Vision and Pattern Recognition Workshops},
  pages = {1--8},
  year = {2024},

}

@inproceedings{wang2018glue,
  title={GLUE: A multi-task benchmark and analysis platform for natural language understanding},
  author={Wang, Alex and Singh, Amanpreet and Michael, Julian and Hill, Felix and Levy, Omer and Bowman, Samuel},
  booktitle={Proceedings of the 2018 EMNLP workshop BlackboxNLP: Analyzing and interpreting neural networks for NLP},
  pages={353--355},
  year={2018}
}

@misc{
liu2020roberta,
title={Ro{\{}BERT{\}}a: A Robustly Optimized {\{}BERT{\}} Pretraining Approach},
author={Yinhan Liu and Myle Ott and Naman Goyal and Jingfei Du and Mandar Joshi and Danqi Chen and Omer Levy and Mike Lewis and Luke Zettlemoyer and Veselin Stoyanov},
year={2020},
}

@inproceedings{zhong2024seeking,
  title={Seeking neural nuggets: Knowledge transfer in large language models from a parametric perspective},
  author={Zhong, Ming and An, Chenxin and Chen, Weizhu and Han, Jiawei and He, Pengcheng},
  booktitle={International Conference on Learning Representations},
  volume={2024},
  pages={25576--25596},
  year={2024}
}

@article{halko2011finding,
  title   = {Finding Structure with Randomness: Probabilistic Algorithms for Constructing Approximate Matrix Decompositions},
  author  = {Halko, Nathan and Martinsson, Per-Gunnar and Tropp, Joel A.},
  journal = {SIAM Review},
  volume  = {53},
  number  = {2},
  pages   = {217--288},
  year    = {2011},
  doi     = {10.1137/090771806}
}
\bibliographystyle{icml2026}

\newpage
\appendix
\onecolumn
 
\section{SAD-LoRA Pseudocode}
\label{app:algorithm}

\begin{algorithm}[h]
\caption{SAD-LoRA Training}
\label{alg:sadlora}
\footnotesize
\begin{minipage}{0.98\linewidth}
\begin{tabbing}
\hspace{1.2em}\=\hspace{1.2em}\=\hspace{1.2em}\=\kill

\textbf{Input:} Pretrained student weights $\{W_0^{(\ell)}\}$; same-architecture reference teacher
$\{W_T^{\mathrm{ref},(\ell)}\}$ in the student parameter space;\\ functional teacher $f_T$ supplying logits; training data $\mathcal{D}$; calibration set $\mathcal{D}_{\mathrm{cal}}$; adapted layers $\mathcal{S}$; ranks $\{r_\ell\}_{\ell\in\mathcal{S}}$;\\ hyperparameters $\alpha,\beta,\tau,\lambda_{\mathrm{KD}}$. \\[0.35em]

\textbf{Offline spectral target construction.} \\
\textbf{for} each adapted layer $\ell\in\mathcal{S}$ \textbf{do} \\
\> Compute the student-space reference update
$\DWT^{(\ell)} \leftarrow W_T^{\mathrm{ref},(\ell)} - W_0^{(\ell)}$. \\
\> Collect centered calibration activations
$X_c^{(\ell)} \in \mathbb{R}^{n\times\din}$ from the frozen student backbone. \\
\> Form the sketch
$Z^{(\ell)} \leftarrow \tfrac{1}{\sqrt{n-1}}\DWT^{(\ell)}X_c^{(\ell)\top}$. \\
\> Compute truncated SVD of $Z^{(\ell)}$ to obtain
$\Utr^{(\ell)}$ and $\sigt_{T,r}^{(\ell)}$. \\
\> Register $\Utr^{(\ell)}$ and $\sigt_{T,r}^{(\ell)}$ as frozen buffers. \\
\textbf{end for} \\[0.35em]

\textbf{Spectral initialization.} \\
\textbf{for} each adapted layer $\ell\in\mathcal{S}$ \textbf{do} \\
\> Initialize $B^{(\ell)} \leftarrow \Utr^{(\ell)}$ and $A^{(\ell)} \leftarrow 0$. \\
\> \textbf{if} $r_\ell$ exceeds the effective spectral rank \textbf{then} \\
\>\> Pad the remaining columns of $B^{(\ell)}$ with a random orthonormal complement. \\
\> \textbf{end if} \\
\textbf{end for} \\[0.35em]

\textbf{Training.} \\
\textbf{for} each minibatch $(x,y)\sim\mathcal{D}$ \textbf{do} \\
\> $z_T \leftarrow f_T(x)$ \hfill \textit{no gradient} \\
\> $z_S \leftarrow f_S(x)$ \\
\> $\mathcal{L}_{\mathrm{task}} \leftarrow \mathrm{TaskLoss}(z_S,y)$ \\
\> $\Lkd \leftarrow
\tau^2\,\mathrm{KL}\!\left(
\mathrm{softmax}(z_T/\tau)\,\|\,\mathrm{softmax}(z_S/\tau)
\right)$ \\

\> \textbf{for} each adapted layer $\ell\in\mathcal{S}$ \textbf{do} \\
\>\> $Q_B^{(\ell)},R_B^{(\ell)} \leftarrow \mathrm{QR}(B^{(\ell)})$ \\
\>\> $\Lalign^{(\ell)}
\leftarrow
1-\tfrac{1}{r_\ell}
\big\|Q_B^{(\ell)\top}\Utr^{(\ell)}\big\|_F^2$ \\
\>\> $\hat{\sigma}^{(\ell)}
\leftarrow
\mathrm{svdvals}\!\left(R_B^{(\ell)}A^{(\ell)}\right)$ \\
\>\> $\Lcoeff^{(\ell)}
\leftarrow
\frac{\big\|\hat{\sigma}^{(\ell)}-\sigt_{T,r}^{(\ell)}\big\|_2^2}
{\big\|\sigt_{T,r}^{(\ell)}\big\|_2^2+\varepsilon}$ \\
\> \textbf{end for} \\

\> $\mathcal{L}
\leftarrow
\mathcal{L}_{\mathrm{task}}
+\lambda_{\mathrm{KD}}\Lkd
+\tfrac{\alpha}{|\mathcal{S}|}\sum_{\ell\in\mathcal{S}}\Lalign^{(\ell)}
+\tfrac{\beta}{|\mathcal{S}|}\sum_{\ell\in\mathcal{S}}\Lcoeff^{(\ell)}$ \\
\> Backpropagate $\mathcal{L}$ and update only
$\{B^{(\ell)},A^{(\ell)}\}_{\ell\in\mathcal{S}}$ with AdamW. \\
\textbf{end for}

\end{tabbing}
\end{minipage}
\end{algorithm}

\section{Proofs of Theoretical Results}\label{app:proofs}
This appendix collects the full proofs of the analytical results in Section~\ref{sec:analysis}. We use the same notation: $\tWT = \DWT\Sigma_x^{1/2} \in \mathbb{R}^{\dout\times\din}$, with compact SVD $\tWT = \Ut\widetilde\Sigma_T\widetilde V_T^\top$ and singular values $\sigt_1 \geq \sigt_2 \geq \cdots \geq 0$. We write $\Utr = \Ut[:,1{:}r]$ for the matrix of top-$r$ left singular vectors and $\widetilde W_T^{(r)} = \Utr\widetilde\Sigma_{T,r}\widetilde V_{T,r}^\top$ for the rank-$r$ truncation.
 
\subsection{Useful Identities}
\label{app:identities}
 
We collect three elementary facts used throughout.
 
\textbf{(F1) Pythagoras for the Frobenius norm.} For any matrix $X \in \mathbb{R}^{\dout\times\din}$ and any orthogonal projector $P \in \mathbb{R}^{\dout\times\dout}$,
\begin{equation*}
    \|X\|_F^2 = \|PX\|_F^2 + \|(I-P)X\|_F^2 ,
\end{equation*}
because $P$ and $I-P$ have orthogonal column ranges and Frobenius inner products of column-orthogonal matrices vanish.
 
\textbf{(F2) Eckart--Young--Mirsky.} The minimum of $\|\tWT - M\|_F^2$ over all rank-$r$ matrices $M$ is $\sum_{i>r}\sigt_i^2$, attained by $\tWT^{(r)} = \Utr\widetilde\Sigma_{T,r}\widetilde V_{T,r}^\top$, whose column space is $\colspan(\Utr)$.
 
\textbf{(F3) Projection of a matrix onto a subspace via $P_B$.} For $B$ with full column rank and $P_B = B(B^\top B)^{-1}B^\top$,
\begin{equation*}
    \|P_B X\|_F^2 = \mathrm{tr}(X^\top P_B X) = \sum_i \|P_B x_i\|_2^2
\end{equation*}
where $x_i$ denotes the $i$-th column of $X$.
 
\subsection{Proof of Proposition~\ref{prop:data_weighted_subspace}}
\label{app:proof_prop1}
 
\begin{proof}
Direct application of Eckart--Young--Mirsky (F2) to $\tWT$. The minimum over rank-$r$ matrices of $\|\tWT - M\|_F^2$ equals $\sum_{i>r}\sigt_i^2$ and is attained by $\tWT^{(r)} = \Utr\widetilde\Sigma_{T,r}\widetilde V_{T,r}^\top$. Its column space is $\colspan(\Utr)$.
\end{proof}
 
\subsection{Proof of Theorem~\ref{thm:decomp}}
\label{app:proof_thm1}
 
\begin{proof}
Let $M = BA\Sigma_x^{1/2} \in \mathbb{R}^{\dout\times\din}$ and let $P_B = B(B^\top B)^{-1}B^\top$ be the orthogonal projector onto $\colspan(B)$. Since $\colspan(M) \subseteq \colspan(B)$, we have $P_B M = M$, hence $(I-P_B)M = 0$.
 
\emph{Step 1: Pythagorean splitting.} Apply (F1) to $X = \tWT - M$:
\begin{align*}
    \mathcal{E}(B,A) \;=\; \|\tWT - M\|_F^2
    &= \|P_B(\tWT - M)\|_F^2 + \|(I-P_B)(\tWT - M)\|_F^2 \\
    &= \|P_B\tWT - M\|_F^2 + \|(I-P_B)\tWT\|_F^2 ,
\end{align*}
where the second equality uses $P_B M = M$ and $(I-P_B)M = 0$. This is the exact two-term projector decomposition.
 
\emph{Step 2: Identifying the rank residual.} Define $\mathcal{R} = \sum_{i>r}\sigt_i^2$. By (F2), for any rank-$r$ orthogonal projector $P$,
\begin{equation*}
    \|(I-P)\tWT\|_F^2 \;=\; \|\tWT\|_F^2 - \|P\tWT\|_F^2 \;=\; \sum_{i\geq 1}\sigt_i^2 - \|P\tWT\|_F^2 .
\end{equation*}
The maximum of $\|P\tWT\|_F^2$ over rank-$r$ projectors is $\sum_{i=1}^r\sigt_i^2$, achieved iff $P$ projects onto $\colspan(\Utr)$. Therefore
\begin{equation*}
    \|(I-P_B)\tWT\|_F^2 \;\geq\; \mathcal{R} ,
\end{equation*}
with equality if and only if $\colspan(B) = \colspan(\Utr)$.
 
\emph{Step 3: Defining $\mathcal{S}(B)$.} Set
\begin{equation*}
    \mathcal{S}(B) \;:=\; \|(I-P_B)\tWT\|_F^2 - \mathcal{R} \;\geq\; 0 .
\end{equation*}
By Step 2, $\mathcal{S}(B) = 0$ iff $\colspan(B) = \colspan(\Utr)$.
 
\emph{Step 4: Defining $\mathcal{C}(B,A)$.} Set
\begin{equation*}
    \mathcal{C}(B,A) \;:=\; \|P_B\tWT - BA\Sigma_x^{1/2}\|_F^2 \;\geq\; 0 ,
\end{equation*}
which is nonnegative by definition.
 
\emph{Step 5: Combining.} Substituting back into Step 1:
\begin{equation*}
    \mathcal{E}(B,A) \;=\; \|P_B\tWT - BA\Sigma_x^{1/2}\|_F^2 + \|(I-P_B)\tWT\|_F^2
    \;=\; \mathcal{C}(B,A) + \mathcal{S}(B) + \mathcal{R} .
\end{equation*}
This is Eq.~\ref{eq:three_term}. The term $\mathcal{R}$ depends only on the rank budget; $\mathcal{S}(B)$ depends only on $\colspan(B)$; $\mathcal{C}(B,A)$ depends on both $B$ and $A$.
 
\emph{Optimality of $A$ given $B$.} For any fixed $B$, $\mathcal{C}(B,A)$ is a quadratic in $A$ minimized at the (pseudo-)inverse solution. Specifically, when $\Sigma_x$ is full rank,
\begin{equation*}
    A^*(B) \;=\; (R_B^\top R_B)^{-1} R_B^\top Q_B^\top \tWT \,(\Sigma_x^{1/2})^{-1} ,
\end{equation*}
so $\mathcal{C}(B,A^*) = 0$. When $\sigt_r>\sigt_{r+1}$, $\mathcal{S}(B)=0$ characterizes the unique optimal output subspace at fixed rank $r$. Under degeneracy, the optimal subspace is not unique, and $\Utr$ denotes the empirical SVD target selected for alignment.
\end{proof}
 
\subsection{Proof of Proposition~\ref{prop:rank_suff}}
\label{app:proof_prop2}
 
\begin{proof}
Fix $\epsilon > 0$. By Theorem~\ref{thm:decomp}, $\mathcal{E}(B,A) \geq \mathcal{R} = \sum_{i>r}\sigt_i^2$ at any rank $r$, with equality attained by $B^* = \Utr$ and the corresponding optimal $A^*$ from Step 5 of Appendix~\ref{app:proof_thm1}. Therefore
\begin{equation*}
    \min_{B,A:\,\mathrm{rank}(B)\leq k} \mathcal{E}(B,A) \;=\; \sum_{i>k}\sigt_i^2 .
\end{equation*}
The smallest $k$ for which this minimum drops below $\epsilon$ is exactly
\begin{equation*}
    r^*_\epsilon(\DWT,\Sigma_x) \;=\; \min\!\left\{k\in\mathbb{N} : \sum_{i>k}\sigt_i^2 < \epsilon\right\} ,
\end{equation*}
which is Eq.~\ref{eq:rank_suff}.
\end{proof}
 
\subsection{Surrogate Relationship: $\Lalign$ and $\mathcal{S}(B)$}
\label{app:lalign_surrogate}
 
We record the precise sense in which $\Lalign$ is a surrogate for $\mathcal{S}(B)$.
 
\begin{lemma}[Zero set]
\label{lem:zero_set}
$\Lalign(B) = 0$ if and only if $\colspan(B) = \colspan(\Utr)$, if and only if $\mathcal{S}(B) = 0$.
\end{lemma}
 
\begin{proof}
Let $B = Q_B R_B$ be the thin QR decomposition. Then
\begin{equation*}
    \Lalign(B) = 1 - \tfrac{1}{r}\|Q_B^\top \Utr\|_F^2 = \tfrac{1}{r}\sum_{i=1}^r \sin^2\theta_i,
\end{equation*}
where $\theta_i$ are the principal angles between $\colspan(Q_B) = \colspan(B)$ and $\colspan(\Utr)$. $\Lalign(B) = 0$ iff every $\theta_i = 0$, iff $\colspan(B) = \colspan(\Utr)$. By Step 2 of Appendix~\ref{app:proof_thm1}, this is also equivalent to $\mathcal{S}(B) = 0$.
\end{proof}
 
\textbf{Why $\Lalign \neq \mathcal{S}$.} The exact spectral form of the subspace miss is
\begin{equation*}
    \|(I-P_B)\tWT\|_F^2 \;=\; \sum_i \sigt_i^2\,\big(1 - \|P_B \widetilde u_i\|_2^2\big) ,
\end{equation*}
which is $\sigt_i^2$-weighted. The alignment loss $\Lalign = \tfrac{1}{r}\sum_{i=1}^r\sin^2\theta_i$ is unweighted (each angle contributes equally) and considers only the top $r$ teacher directions. Thus $\Lalign$ shares its zero set with $\mathcal{S}$ (Lemma~\ref{lem:zero_set}) but does not match its level sets. In practice this is desirable: $\Lalign$ is differentiable in $B$, computable from a single thin QR plus one $r\times r$ matmul, and numerically stable, whereas a $\sigt_i^2$-weighted version would require computing the full $\widetilde u_i$ basis at every step.
 
\subsection{Coefficient Loss: Data-Weighted vs.\ Unweighted Forms}
\label{app:coeff_dw}
 
The principled within-subspace coefficient term implied by $\mathcal{C}(B,A) = \|P_B\tWT - BA\Sigma_x^{1/2}\|_F^2$ depends on the singular values of $BA\Sigma_x^{1/2}$, since $P_B\tWT$ has top-$r$ singular values bounded by those of $\tWT$ projected onto $\colspan(B)$. The matching loss
\begin{equation*}
    \Lcoeff^{\mathrm{dw}}(B,A) \;=\; \frac{\big\|\sigma(R_B A \Sigma_x^{1/2}) - \sigt_{T,r}\big\|_2^2}{\big\|\sigt_{T,r}\big\|_2^2 + \varepsilon}
\end{equation*}
is therefore the data-weighted form. Computing it requires applying $\Sigma_x^{1/2}$, which is not stored; the cheapest equivalent is right-multiplication by the calibration sketch $X_c$ (which has the same nonzero singular values up to the $\sqrt{n-1}$ factor), at cost $O(rn\din)$ per layer per step. We instead use
\begin{equation*}
    \Lcoeff(B,A) \;=\; \frac{\big\|\sigma(R_B A) - \sigt_{T,r}\big\|_2^2}{\big\|\sigt_{T,r}\big\|_2^2 + \varepsilon} ,
\end{equation*}
which costs $O(r^2\din)$ per layer per step (one $r\times\din$ SVD) and equals $\Lcoeff^{\mathrm{dw}}$ when $\Sigma_x = I$. In the anisotropic case the two losses differ; both share the trivial zero $A=0,\sigt_{T,r}=0$ but not their general level sets. Because our ablations identify $\Lcoeff$ as auxiliary rather than load-bearing (Section~\ref{sec:ablation_rank_analysis}), the practical cost of using the proxy is small. We confirmed in preliminary runs on SST-2, MRPC, and CoLA that swapping $\Lcoeff$ for $\Lcoeff^{\mathrm{dw}}$ produces no consistent improvement.
 
\section{Additional Experimental Details}
\label{app:experimental_details}
 
\subsection{Hyperparameters}
\label{app:hyperparams}
 
Table~\ref{tab:hyperparams} collects all SAD-LoRA hyperparameters. Optimization, calibration size, and energy threshold values are held fixed across all six GLUE tasks and all five ranks; only the LoRA rank itself varies in our rank-trajectory experiments.
 
\begin{table}[h]
\centering
\caption{SAD-LoRA hyperparameters used throughout. Values are matched across baselines wherever they share a parameter (optimizer, LR schedule, KD temperature).}
\label{tab:hyperparams}
\small
\begin{tabular}{lll}
\toprule
Parameter & Value & Notes \\
\midrule
\multicolumn{3}{l}{\textit{SAD-LoRA-specific}} \\
\midrule
Alignment weight $\alpha$ & $0.2$ & Eq.~\ref{eq:sadlora_objective} \\
Coefficient weight $\beta$ & $0.05$ & Eq.~\ref{eq:sadlora_objective}; $0$ in SAD-LoRA-Align \\
Distillation temperature $\tau$ & $4.0$ & Eq.~\ref{eq:kd_loss} \\
KD weight $\lambda_{\mathrm{KD}}$ & $1.0$ & Eq.~\ref{eq:sadlora_objective} \\
Calibration size $n$ & $1024$ & Phase 2 (Section~\ref{sec:method}) \\
Spectral energy threshold $\delta$ & $0.01$ & Prop.~\ref{prop:rank_suff} \\
Stored ranks $r_{\max}$ & $64$ & Cap on per-layer SVD \\
Numerical $\varepsilon$ & $10^{-8}$ & Eq.~\ref{eq:coefficient_loss} \\
\midrule
\multicolumn{3}{l}{\textit{Optimizer (matched across all methods)}} \\
\midrule
Optimizer & AdamW & Decoupled weight decay \\
Weight decay & $0.01$ & --- \\
Learning rate (LoRA) & $5\!\times\!10^{-4}$ & --- \\
Warmup fraction & $6\%$ & Linear, then linear decay \\
Batch size & $32$ & SST-2, QNLI; $16$ for small tasks \\
Max sequence length & $128$ & Standard for GLUE \\
Mixed precision & fp16 & Spectral ops in fp32 \\
Seeds & $\{42,123,2024\}$ & --- \\
\bottomrule
\end{tabular}
\end{table}
 
\subsection{Optimization}
We use AdamW with weight decay $0.01$, learning rate $5\!\times\!10^{-4}$ for LoRA parameters, linear warmup over the first $6\%$ of steps, and linear decay thereafter. Mixed-precision training is enabled, with all spectral operations (QR, SVD, computation of $\Lalign$ and $\Lcoeff$) executed in float32 for numerical stability.
 
\subsection{Spectral Target Construction}
For each adapted layer, Phase~1 computes $\DWT^{(\ell)}$ from a reference RoBERTa-base teacher fine-tuned on the same task with the same data split. Phase~2 uses $n=1024$ calibration samples drawn from the training set; activations are mean-centered before forming $Z^{(\ell)}$. Truncated SVD uses randomized SVD \citep{halko2011finding} when $\dout > 256$.
 
\subsection{Per-layer Rank}
We use a uniform rank $r$ across layers in our main results, matching the convention of prior baselines and isolating the effect of the SAD-LoRA losses from rank-allocation choices. Proposition~\ref{prop:rank_suff} suggests that allocating ranks per layer using the spectral tail of $\widetilde W_T^{(\ell)}$ (with $\delta=0.01$) could yield further gains; preliminary results indicate small but positive effects on parameter efficiency, but a careful per-layer ablation is left to future work.
 
\subsection{Subspace Alignment Across Tasks and Ranks}
\label{app:alignment_scores}
 
Table~\ref{tab:alignment_scores} reports the final post-training subspace alignment $A^{(\ell)} = \tfrac{1}{r}\|Q_B^{(\ell)\top}\Utr^{(\ell)}\|_F^2$, averaged over adapted layers and three seeds, for SAD-LoRA at each rank. Higher is better; $1.0$ means $\colspan(B)$ exactly recovers the data-weighted teacher subspace. The pattern mirrors the GLUE results: tasks with the strongest task-specific spectral structure (RTE, MRPC, CoLA, STS-B) reach near-perfect alignment even at $r=4$, while QNLI and SST-2 — which are nearly saturated by the pretrained backbone and therefore have flatter $\widetilde W_T$ spectra — require larger $r$ to recover the same alignment.
 
\begin{table}[h]
\centering
\caption{Mean post-training subspace alignment $\tfrac{1}{r}\|Q_B^\top\Utr\|_F^2$ for SAD-LoRA across tasks and ranks (averaged over adapted layers and three seeds). Bold entries indicate $\geq 0.95$ alignment.}
\label{tab:alignment_scores}
\small
\begin{tabular}{lccccc}
\toprule
Task & $r{=}1$ & $r{=}2$ & $r{=}4$ & $r{=}8$ & $r{=}16$ \\
\midrule
RTE   & $\mathbf{1.00}$ & $\mathbf{1.00}$ & $\mathbf{0.99}$ & $\mathbf{0.99}$ & $\mathbf{0.99}$ \\
MRPC  & $0.87$ & $0.90$ & $0.94$ & $\mathbf{0.96}$ & $\mathbf{0.97}$ \\
CoLA  & $0.82$ & $0.90$ & $0.93$ & $\mathbf{0.95}$ & $\mathbf{0.97}$ \\
SST-2 & $0.75$ & $0.84$ & $0.88$ & $0.91$ & $0.94$ \\
QNLI  & $0.69$ & $0.78$ & $0.84$ & $0.88$ & $0.90$ \\
STS-B & $0.68$ & $0.82$ & $0.90$ & $0.94$ & $\mathbf{0.96}$ \\
\bottomrule
\end{tabular}
\end{table}
 
\subsection{Per-Task Ablation Across Ranks}
\label{app:ablation_per_task}
 
Fig.~\ref{fig:ablation_full} extends Fig.~\ref{fig:wins_r8} by reporting all three SAD-LoRA variants — full SAD-LoRA, SAD-LoRA-Align ($\beta{=}0$), and SAD-LoRA-Coeff ($\alpha{=}0$) — at both $r{=}4$ and $r{=}8$ for every GLUE task. The deltas are computed against the strongest non-SAD baseline (KD-LoRA, Logit-MSE, NN-Init, or PiSSA-Init) for each task and rank. Two patterns are visible: (i) SAD-LoRA-Coeff is the most damaging ablation, with large negative deltas on RTE at both ranks ($-6.1$ at $r{=}4$, $-4.1$ at $r{=}8$), confirming that subspace alignment carries the load; (ii) SAD-LoRA-Align is competitive across all tasks at $r{=}8$, while the full objective adds a measurable advantage on SST-2 (+0.5 at both ranks). MRPC is the consistent exception, where pretrained-weight principal structure (PiSSA-Init) remains a stronger inductive bias than the data-weighted teacher update.
 
\begin{figure}[h]
\centering
\includegraphics[width=\linewidth]{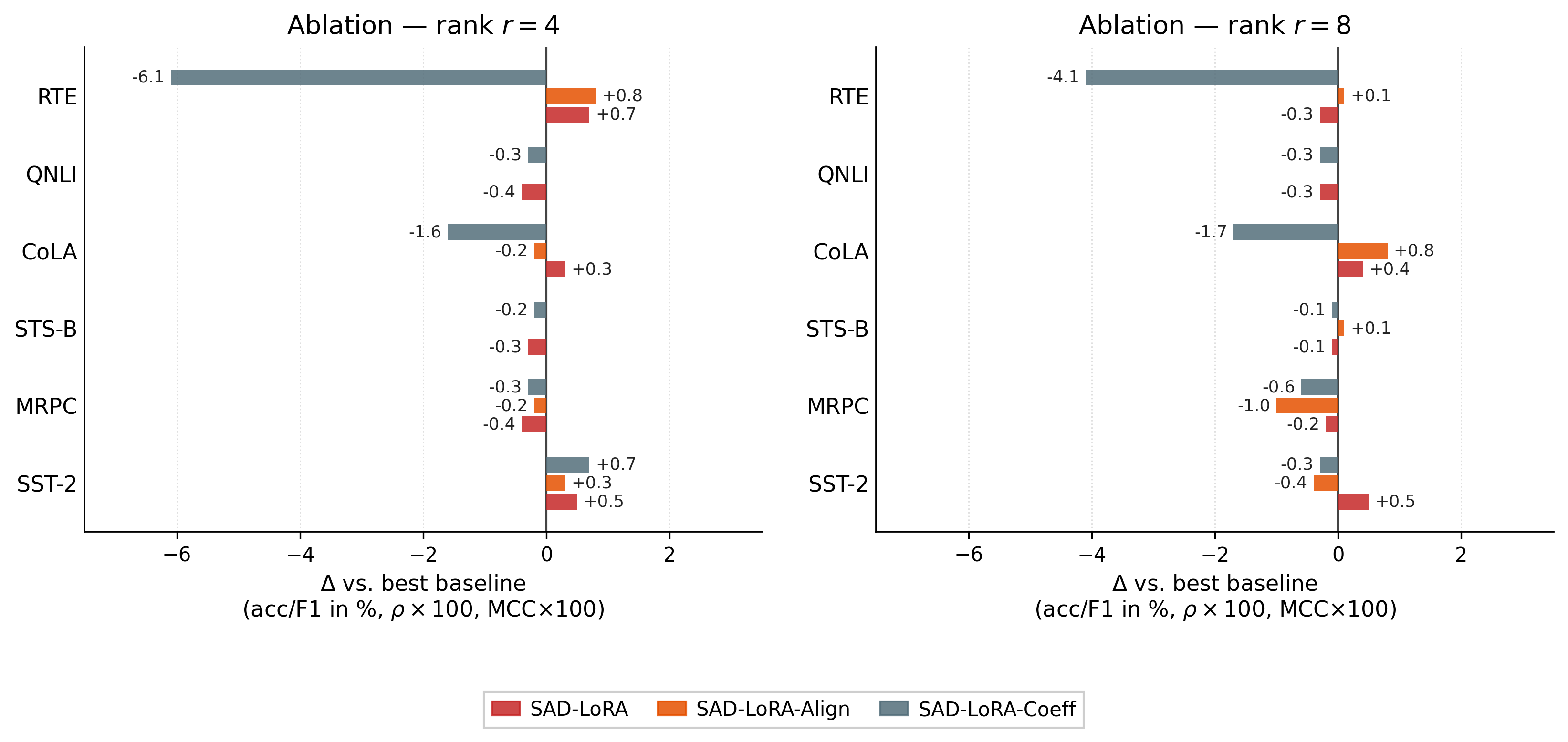}
\caption{Per-task ablation: $\Delta$ versus the best non-SAD baseline at $r{=}4$ (left) and $r{=}8$ (right) for each SAD-LoRA variant. Positive values indicate improvement over the strongest baseline. Removing $\Lalign$ (SAD-LoRA-Coeff) is consistently more damaging than removing $\Lcoeff$ (SAD-LoRA-Align).}
\label{fig:ablation_full}
\end{figure}
 
\subsection{Reproducibility}
Code is available at https://github.com/OmerTariq-KAIST/sad-lora. Random seeds are $\{42,123,2024\}$. Each GLUE configuration was trained on a single RTX6000Pro~98GB GPU; total wall-clock time for the GLUE result table is approximately $48$~GPU-hours.

\end{document}